\RequirePackage{fix-cm}
\documentclass[twocolumn]{svjour3} 
\smartqed  
\usepackage{graphicx}

\usepackage[latin1]{inputenc}
\usepackage[perpage,para,symbol*]{footmisc}
\usepackage{amsfonts}
\usepackage[usenames,dvipsnames,svgnames,table]{xcolor}
\usepackage{amsmath}
\usepackage{mathrsfs}
\usepackage{latexsym}
\usepackage{amssymb}
\usepackage{hyperref}       
\usepackage{url}            
\usepackage{booktabs}       
\usepackage{amsfonts}       
\usepackage{nicefrac}       
\usepackage{microtype}      
\usepackage{xcolor}
\usepackage{algorithm}
\usepackage{algorithmic}
\usepackage{graphicx}
\usepackage{float} 
\usepackage{verbatim}
\usepackage{enumerate}
\usepackage{epsfig}
\usepackage{graphicx}
\usepackage{array}
\usepackage{fancyhdr}
\usepackage{bbm}
\usepackage{mathrsfs}
\usepackage{dsfont}
\usepackage{hyperref}
\usepackage[numbers]{natbib} 
\allowdisplaybreaks

\newtheorem{thm}{Theorem}

\newtheorem{cor}[thm]{Corollary}
\newtheorem{rmk}[thm]{Remark}
\newtheorem{prp}[thm]{Proposition}

\newcommand{\anti}[2]{A_{#2}({#1})}

\usepackage{verbatim}

\begin{document}

\title{Antithetic and Monte Carlo kernel estimators for partial rankings}

\author{M. Lomeli  \and
             M. Rowland \and
            A. Gretton \and Z. Ghahramani 
}


\institute{
M. Lomeli \at  Computational and Biological Lab, University of Cambridge.  \email{maria.lomeli@eng.cam.ac.uk }
\and
M. Rowland \at Department of Pure Mathematics and Mathematical Statistics, University of Cambridge. \email{mr504@cam.ac.uk}
\and 
A. Gretton \at
Gatsby Computational Neuroscience Unit, University College London. \email{arthur.gretton@gmail.com}
\and
Z. Ghahramani \at
Computational and Biological Lab, University of Cambridge and Uber AI Labs. \email{zoubin@eng.cam.ac.uk}
}

\date{Received: date / Accepted: date}

\maketitle

\begin{abstract}
In the modern age, rankings data is ubiquitous and it is useful for a  variety of applications such as recommender systems, multi-object tracking and preference learning. However, most rankings data encountered in the real world is incomplete, which prevents the direct application of existing modelling tools for complete rankings. Our contribution is a novel way to extend kernel methods for complete rankings to partial rankings, via consistent Monte Carlo estimators for Gram matrices: matrices of kernel values between pairs of observations.
 We also present a novel variance reduction scheme based on an antithetic variate construction between permutations to obtain an improved estimator for the Mallows kernel.  The corresponding antithetic kernel estimator has lower variance and we demonstrate empirically that it has a better performance in a variety of Machine Learning tasks.
 Both kernel estimators are based on extending kernel mean embeddings to the embedding of a set of full rankings consistent with an observed partial ranking. They form a computationally tractable alternative to previous approaches for partial rankings data. 
  An overview of the existing kernels and metrics for permutations is also provided. 
\end{abstract}

\section{Motivation}
\label{motiv}
 Permutations play a fundamental role in statistical modelling and machine learning applications involving rankings and preference data. A ranking over a set of objects can be encoded as a permutation, hence, kernels for permutations are useful in a variety of machine learning applications involving rankings. Applications include recommender systems, multi-object tracking and preference learning. It is of interest to construct a kernel in the space of the data in order capture similarities between datapoints and thereby influence the pattern of generalisation. Kernels are used in many machine learning methods. For instance, a kernel input is required for the maximum mean discrepancy (MMD) two sample test~\citep{GreBorMalSchoSmo12}, kernel principal component analysis (kPCA)~\citep{SchSmoMull99}, support vector machines~\citep{BosGuyVap92, CorVap95}, Gaussian processes (GPs)~\citep{RasWil2006a} and agglomerative clustering~\citep{DudHar73}, among others. 
  
Our main contributions are: (i) A novel and computationally tractable way to deal with incomplete or partial rankings by first representing the marginalised kernel \citep{Haussler99} as a kernel mean embedding of a set of full rankings consistent with an observed partial ranking. We then propose two estimators that can be represented as the corresponding empirical mean embeddings: (ii) A Monte Carlo kernel estimator that is based on sampling independent and identically distributed rankings from the set of consistent full rankings given an observed partial ranking; (iii) An antithetic variate construction for the marginalised Mallows kernel that gives a lower variance estimator for the kernel Gram matrix. The Mallows kernel has been shown to be an expressive kernel; in particular, ~\citet{Mania16} show that the Mallows kernel is an example of a universal and characteristic kernel, and hence it is a useful tool to distinguish samples from two different distributions, and it achieves the Bayes risk when used in kernel-based classification/regression~\citep{SriFukLan11}. \citet{JiaoVert15} have proposed a fast approach for computing the Kendall marginalised kernel, however, this kernel is not characteristic~\citep{Mania16}, and hence has limited expressive power.

The resulting estimators are used for a variety of kernel machine learning algorithms in the experiments. We present comparative simulation results demonstrating the efficacy of the proposed estimators for an agglomerative clustering task, a hypothesis test task using the maximum mean discrepancy (MMD)~\citep{GreBorMalSchoSmo12}  and a Gaussian process classification task. For the latter, we extend some of the existing methods in the software library GPy~\citep{gpy2014}.

 Since the space of permutations is an example of a discrete space, with a non-commutative group structure, the corresponding reproducing kernel Hilbert spaces (RKHS) have only recently being investigated; see~\citet{KonHow07},~\citet{FukSriGreScho09},~\citet{KonBar10},~\citet{JiaoVert15} and \citet{Mania16}. We provide an overview of the connection between kernels and  certain semimetrics when working on the space of permutations. This connection allows us to obtain kernels from given semimetrics or semimetrics from existing kernels. We can combine these semimetric-based kernels to obtain novel, more expressive kernels which can be used for the proposed Monte Carlo kernel estimator. 

\section{Definitions}\label{sec:def}

We first briefly introduce the theory of permutation groups. A particular application of permutations is to use them to represent rankings; in fact, there is a natural one-to-one relationship between rankings of $n$ items and permutations. For this reason, we sometimes use ranking and permutation interchangeably. In this section, we state some mathematical definitions to formalise the problem in terms of the space of permutations.

Let $\left[n\right]=\left\{1,2,\hdots,n\right\}$ be a set of indices for $n$ items, for some $n \in \mathbb{N}$.
Given a ranking of these $n$ items, we use the notation $\succ$ to denote the ordering of the items induced by the ranking, so that for distinct $i, j \in \left[ n \right]$, if $i$ is preferred to $j$, we will write $i \succ j$. Note that for a full ranking, the corresponding relation $\succ$ is a total order on $\{1,\ldots, n\}$.

We now outline the correspondence between rankings on $\left[ n \right]$ and the  permutation group $S_n$ that we use throughout the paper. In words, given a full ranking of $[n]$, we will associate it with the permutation $\sigma \in S_n$ that maps each ranking position $1,\ldots,n$ to the correct object under the ranking. More mathematically, given a ranking $a_1 \succ \cdots \succ a_n$ of $\left[ n \right]$, we may associate it with the permutation $\sigma \in S_n$ given by $\sigma(j) = a_j$ for all $j =1,\ldots,n$. For example, the permutation corresponding to the ranking on $[3]$ given by $2\succ 3\succ 1$, corresponds to the permutation $\sigma \in S_3$ given by $\sigma(1)=2,\sigma(2)=3,\sigma(3)=1$.
This correspondence allows the literature relating to kernels on permutations to be leveraged for problems involving the modelling of ranking data.

In the next section, we will review some of the semimetrics on $S_n$ that can serve as building blocks for the construction of more expressive kernels.

\subsection{Metrics for permutations and properties}
\label{defs}

\begin{definition}
Let $\mathcal{X}$ be any set and $d:\mathcal{X}\times \mathcal{X}\rightarrow \mathbb{R}$ is a function, which we write $d(x,y)$ for every $x,y \in \mathcal{X}$. Then $d$ is a \emph{semimetric} if it satisfies the following conditions, for every $x,y\in\mathcal{X}$ \citep{Dud02}:	
	\begin{enumerate}[i)]
		\item $d(x,y)=d(y,x)$, that is, $d$ is a symmetric function.
		\item $d(x,y)=0$ if and only if $x=y$.		
		
		A \emph{semimetric} is a \emph{metric} if it satifies:

		\item $d(x,z)\leq d(x,y)+d(y,z)$ for every $x,y, z \in \mathcal{X}$, that is, $d$ satisfies the triangle inequality.
	\end{enumerate}
\end{definition}

The following  are some examples of semimetrics on the space of permutations $S_n$ \citep{Dia88}. All semimetrics in bold have the additional property of being of negative type. Theorem \ref{chon}, stated below, shows that negative type semimetrics are closely related to kernels.
 
\label{metrics}
\begin{enumerate}[1)]
	\item \emph{Spearman's footrule}. 
	
	$d_1(\sigma,\sigma') = \sum_{i=1}^n |\sigma(i) - \sigma'(i)|= \|\sigma-\sigma'\|_1$.
	\item \emph{\textbf{Spearman's rank correlation.}}
	
	 $\mathbf{d_2}(\sigma,\sigma')= \sum_{i=1}^n (\sigma(i) -\sigma'(i))^2= \|\sigma-\sigma'\|^2_2$. 
	\item \emph{\textbf{Hamming distance}.} 
	
	$d_H(\sigma,\sigma') = \# \{i | \sigma(i) \not= \sigma'(i) \}.$ It can also be defined as the minimum number of substitutions required to change one permutation into the other. 
	\item \emph{Cayley distance.}
	
	 $d_C(\sigma, \sigma') = \sum_{j=1}^{n-1}X_j(\sigma\circ(\sigma')^{-1})$, 
	 
	 where the composition operation of the permutation group $S_n$ is denoted by $\circ$ and $X_j(\sigma\circ (\sigma')^{-1})= 0$ if $j$ is the largest item in its cycle and is equal to 1 otherwise \citep{IruCalLoz16}. It is also equal to the minimum number of pairwise transpositions taking $\sigma$ to $\sigma'$. Finally, it can also be shown to be equal to $n-C(\sigma\circ(\sigma')^{-1})$ where $C(\eta)$ is the number of cycles in $\eta$.
	\item \emph{\textbf{Kendall distance.}} 
	
	$\mathbf{d_{\tau}(\sigma, \sigma')} =  n_d(\sigma,  \sigma')$,
	
	 where $n_d(\sigma, \sigma')$ is the number of discordant pairs for the permutation pair $(\sigma, \sigma')$. It can also be defined as the minimum number of pairwise adjacent transpositions taking $\sigma^{-1}$ to $(\sigma')^{-1}$.
	\item \emph{$l_p$ distances.} $d_p(\sigma, \sigma') = \left(\sum_{i=1}^n |\sigma(i) - \sigma'(i)|^p\right)^{\frac{1}{p}}= \|\sigma-\sigma'\|_p$ with $p\geq 1$.
	\item \emph{$l_{\infty}$ distance.} $d_{\infty}(\sigma, \sigma') =\stackrel{\text{max}}{_{1\leq i \leq n}}|\sigma(i) - \sigma'(i)|= \|\sigma-\sigma'\|_{\infty}$.

\end{enumerate}

\begin{definition} A \emph{semimetric} is said to be of \emph{negative type} if for all $n\geq 2$, $x_1,\hdots,x_n\in \mathcal{X}$ and $\alpha_1,\hdots, \alpha_n\in \mathbb{R}$ with $\sum_{i=1}^n \alpha_i=0$, we have
\begin{align}\label{negtyp}
\sum_{i=1}^n\sum_{j=1}^n\alpha_i\alpha_j d(x_i,x_j)\leq 0.
\end{align}
In general, if we start with a Mercer kernel for permutations, that is, a symmetric and positive definite function $k : S_n \times S_n \rightarrow \mathbb{R}$, the following expression gives a semimetric $d$ that is of negative type	
\begin{align}\label{def:semimetneg}
d(\sigma,\sigma')^2&= k(\sigma,\sigma)+k(\sigma',\sigma')-2k(\sigma,\sigma').
\end{align}
\end{definition} 
A useful characterisation of semimetrics of negative type is given by the following theorem, which states a connection between negative type metrics and a Hilbert space feature representation or feature map $\phi$.

\begin{theorem} {\citep{BerChrRes84}. }\label{chon} A semimetric $d$ is of negative type if and only if there exists a Hilbert space $\mathcal{H}$ and an injective map $\phi:\mathcal{X}\rightarrow \mathcal{H}$ such that $\forall x,x' \in \mathcal{X}$,
	$d(x,x')=\|\phi(x)-\phi(x')\|_{\mathcal{H}}^2$.
\end{theorem}
Once the feature map from Theorem~\ref{chon} is found, we can directly take its inner product to construct a kernel. For instance, \citet{JiaoVert15} propose an explicit feature representation for Kendall kernel given by
\[
\displaystyle{
\Phi(\sigma)=\left(\frac{1}{\sqrt{\binom{n}{2}}}\left[\mathbb{I}_{\left\{\sigma(i)>\sigma(j)\right\}}-\mathbb{I}_{\left\{\sigma(i)<\sigma(j)\right\}}\right]\right)_{1\leq i<j\leq n}
} \, .\] 

They show that the inner product between two such features is a positive definite kernel. The corresponding metric, given by Kendall distance, can be shown to be the square of the norm of the difference of feature vectors. Hence, by Theorem~\ref{chon}, it is of negative type.

Analogously, \citet{Mania16} propose an explicit feature representation for the Mallows kernel, given by

\noindent$\displaystyle{
\Phi(\sigma)=\left(\frac{1-\exp{(-v)}}{2}\right)^{\frac{1}{2}\binom{n}{2}}\left(\frac{1-\exp{(-v)}}{1+\exp{(-v)}}\right)^{\frac{r}{2}}\prod_{i=1}^r\bar{\Phi}(\sigma)_{s_i}
}$ 
\noindent where $\bar{\Phi}(\sigma)_{s_i}=2\mathbb{I}_{\left\{\sigma(a_i)<\sigma(b_i)\right\}-1}$ when $s_i=(a_i,b_i)$ and $\bar{\Phi}(\sigma)_{\emptyset}=2^{\frac{1}{2}\binom{n}{2}}(1+\exp{(-v)})^{\frac{1}{2}\binom{n}{2}}$.

In the following proposition, an explicit feature representation for the Hamming distance is introduced and we show that it is a distance of negative type.
\begin{proposition}\label{hamm} The Hamming distance is of negative type with 
\begin{align}
d_H(\sigma, \sigma') &=\frac{1}{2} \text{Trace}\left[\left(\Phi(\sigma)-\Phi(\sigma')\right)\left(\Phi(\sigma)-\Phi(\sigma')\right)^T\right]
\end{align} where the corresponding feature representation is a matrix given by 
\[ \Phi(\sigma)=\left( \begin{array}{ccc}
	\mathbb{I}_{\left\{\sigma(1)=1\right\}}&\hdots&\mathbb{I}_{\left\{\sigma(n)=1\right\}}  \\
	\mathbb{I}_{\left\{\sigma(1)=2\right\}}&\hdots&\mathbb{I}_{\left\{\sigma(n)=2\right\}} \\
	\vdots&\hdots&\vdots\\
	\mathbb{I}_{\left\{\sigma(1)=n\right\}}&\hdots&\mathbb{I}_{\left\{\sigma(n)=n\right\}} \end{array} \right).\]
\end{proposition}
\begin{proof}{The Hamming distance can be written as a square difference of indicator functions in the following way}
	\begin{align}
	d_H(\sigma, \sigma') &= \# \{i | \sigma(i) \not= \sigma'(i) \}\nonumber\\
	&=\frac{1}{2} \sum_{i=1}^n\sum_{\ell=1}^n\biggl(\mathbb{I}_{\left\{\sigma(i)=\ell\right\}}-\mathbb{I}_{\left\{\sigma'(i)=\ell\right\}}\biggr)^2\nonumber \\
	\intertext{where each indicator is one whenever the given entry of the permutation is equal to the corresponding element of the identity element of the group. Let the $\ell$-th feature vector be $\phi_{\ell}(\sigma)=\left(\mathbb{I}_{\left\{\sigma(1)=\ell\right\}},\hdots,\mathbb{I}_{\left\{\sigma(n)=\ell\right\}}\right)$, then}
	&=\frac{1}{2} \sum_{\ell=1}^n(\phi_{\ell}(\sigma)-\phi_{\ell}(\sigma'))^T(\phi_{\ell}(\sigma)-\phi_{\ell}(\sigma'))\nonumber\\
	&= \frac{1}{2} \sum_{\ell=1}^n\|\phi_{\ell}(\sigma)-\phi_{\ell}(\sigma')\|^2\nonumber\\
	&=\frac{1}{2} \text{Trace}\left[\left(\Phi(\sigma)-\Phi(\sigma')\right)\left(\Phi(\sigma)-\Phi(\sigma')\right)^T\right]\nonumber.
	\end{align}
	This is the trace of  the difference of the product of the feature matrices $\Phi(\sigma)-\Phi(\sigma')$, where the  difference of feature matrices is given by
	\[ \left( \begin{array}{ccc}
	\mathbb{I}_{\left\{\sigma(1)=1\right\}}-\mathbb{I}_{\left\{\sigma'(1)=1\right\}}&\hdots&\mathbb{I}_{\left\{\sigma(n)=1\right\}}-\mathbb{I}_{\left\{\sigma'(n)=1\right\}}  \\
	\mathbb{I}_{\left\{\sigma(1)=2\right\}}-\mathbb{I}_{\left\{\sigma'(1)=2\right\}}&\hdots&\mathbb{I}_{\left\{\sigma(n)=2\right\}}-\mathbb{I}_{\left\{\sigma'(n)=2\right\}} \\
	\vdots&\vdots&\vdots\\
	\mathbb{I}_{\left\{\sigma(1)=n\right\}}-\mathbb{I}_{\left\{\sigma'(1)=n\right\}}&\hdots&\mathbb{I}_{\left\{\sigma(n)=n\right\}}-\mathbb{I}_{\left\{\sigma'(n)=n\right\}}   \end{array} \right).\]
	This is the square of the usual Frobenius norm for matrices, so by Theorem~\ref{chon}, the Hamming distance is of negative type.

\end{proof}

 Another example is Spearman's rank correlation, which is a semimetric of negative type since it is the square of the usual Euclidean distance \citep{BerChrRes84}.

The two alternative definitions given for some of the distances in the previous examples are handy from different perspectives. One is an expression in terms of either an injective or non-injective feature representation, while the other is in terms of the minimum number of operations to change one permutation to the other.  Other distances can be defined in terms of this minimum number of operations, they are called \emph{editing metrics} \citep{DezDez09}. Editing metrics are useful from an algorithmic point of view whereas metrics defined in terms of feature vectors are useful from a theoretical point of view. Ideally, having a particular metric in terms of both algorithmic and theoretical descriptions gives a better picture of which are the relevant characteristics of the permutation that the metric takes into account. For instance, Kendall and Cayley distances algorithmic descriptions correspond to the bubble and quick sort algorithms respectively \citep{Knu98}. 	

\begin{figure}[h]	
\includegraphics[scale=0.4]{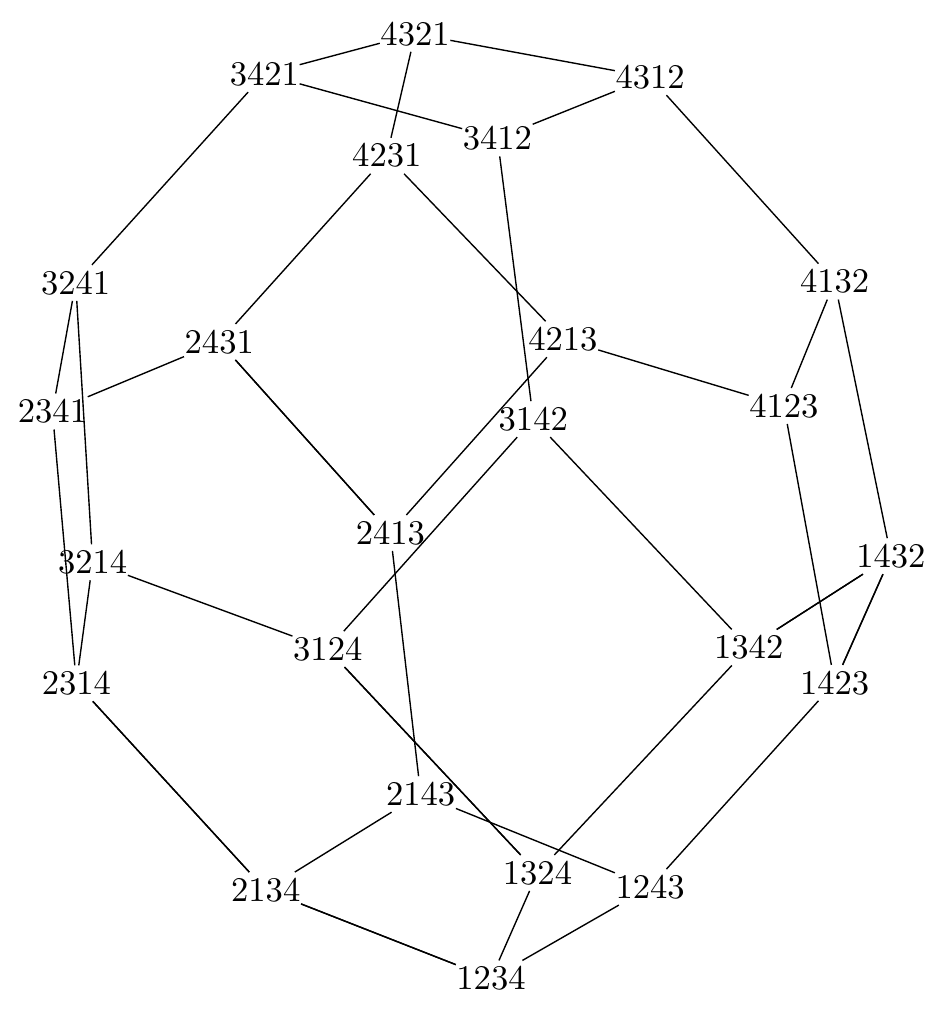}
\includegraphics[scale=0.4]{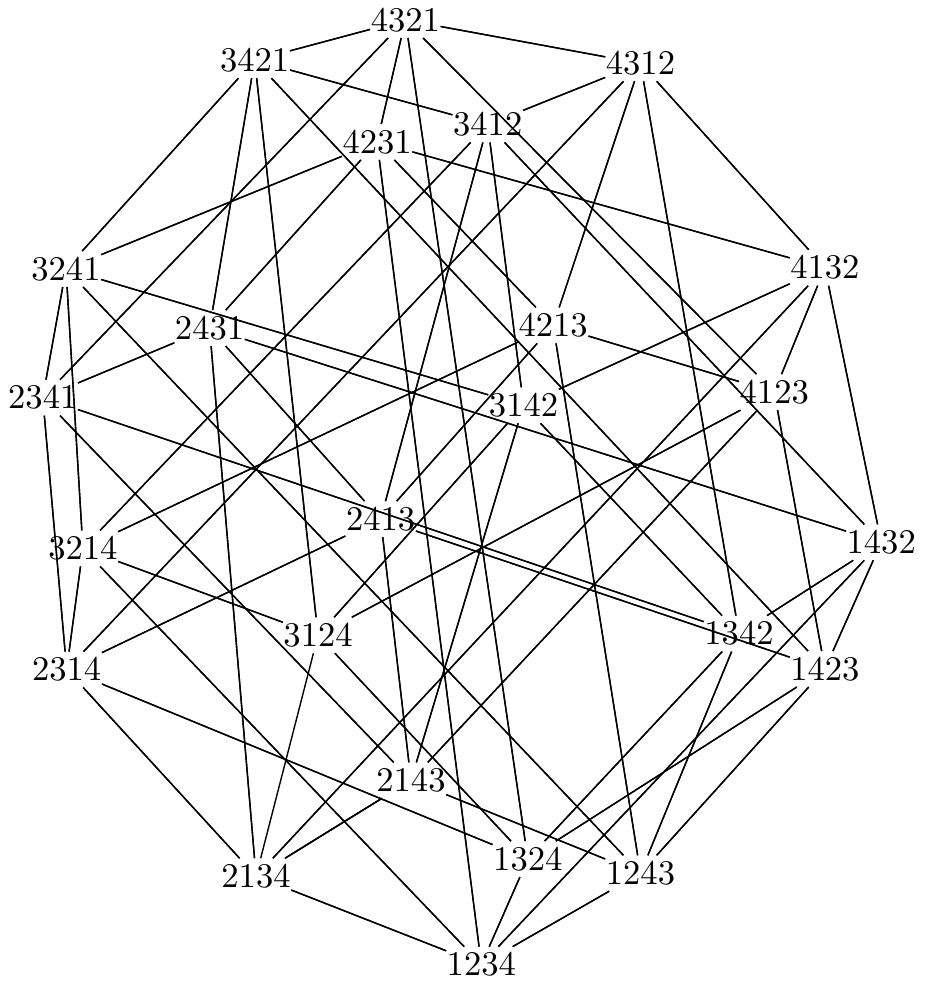}\caption{Kendall and Cayley distances for permutations of $n=4$. There is an edge between two permutations in the graph if they differ by one adjacent or non-adjacent transposition, respectively.}
\end{figure}
Another property shared by most of the semimetrics in the examples is the following 

\begin{definition} Let $\sigma_1,\sigma_2\in S_n$, $(S_n,\circ)$ denote the symmetric group of degree n with the composition operation, a \emph{right-invariant} semimetric  \citep{Dia88} satisfies
\begin{align}\label{rightindef}
d(\sigma_1,\sigma_2)&= d(\sigma_1\circ\eta,\sigma_2\circ \eta) \hspace{2mm}\forall  \ \sigma_1, \sigma_2, \eta \in S_n.
\end{align}
In particular, if we take $\eta=\sigma_1^{-1}$ then $d(\sigma_1,\sigma_2)=d(e,\sigma_2\circ \sigma_1^{-1})$, where $e$ corresponds to the identity element of the permutation group.
\end{definition}
This property is inherited by the \emph{distance-induced kernel} from Section~\ref{Exampleskern}, Example~7. This symmetry is analogous to translation invariance for kernels defined in Euclidean spaces.
Intuitively, if we start with one of the basic semimetrics, we could combine one or more to capture more diverse characteristics. In the next sections, we describe how, if we exploit the relationship between semimetrics and kernels and we combine kernels due to the available operations that still render a kernel, we can produce more expressive kernels. 
\subsection{Kernels for $S_n$\label{Exampleskern}}

If we specify a symmetric and positive definite function or kernel $k$, it corresponds to defining an implicit feature space representation of a ranking data point. The well-known \textit{kernel trick} exploits the implicit nature of this representation by performing computations with the kernel function explicitly, rather than using inner products between feature vectors in high or even  infinite dimensional space. Any symmetric and positive definite function uniquely defines an underlying Reproducing Kernel Hilbert Space (RKHS), see the supplementary material Appendix~\ref{app:rkhs} for a brief overview about the RKHS.
Some examples of kernels for permutations are the following
\begin{table}
\
\end{table}
\begin{enumerate}
\item The \emph{Kendall kernel} \citep{JiaoVert15} is given by

 $\displaystyle{k_\tau(\sigma, \sigma^\prime) = \frac{n_c(\sigma, \sigma^\prime) - n_d(\sigma, \sigma^\prime)}{\binom{d}{2}}}$,\\
 \noindent
 where $n_c(\sigma, \sigma^\prime)$ and $ n_d(\sigma, \sigma^\prime)$ denote the number of concordant and discordant pairs between $\sigma$ and $\sigma^\prime$ respectively.
\item  The \emph{Mallows kernel} \citep{JiaoVert15} is given by

 $	\displaystyle{k^\nu(\sigma, \sigma^\prime) = \exp(-\nu n_d(\sigma, \sigma^\prime))}$.
\item The \emph{Polynomial kernel of degree m} \citep{Mania16}, is given by

 $\displaystyle{ k_{P}^{(m)}(\sigma, \sigma^\prime) = (1 + k_{\tau}(\sigma, \sigma^\prime))^m}$.
\item The \emph{Hamming kernel} is given by

$\displaystyle{k_H(\sigma, \sigma^\prime) =  \text{Trace}\left[\left(\Phi(\sigma)\Phi(\sigma'\right)^T\right]}$.
\item An \emph{exponential semimetric kernel} is given by

$\displaystyle{k_d(\sigma, \sigma^\prime) =  \exp\left\{-\nu d(\sigma, \sigma^\prime)\right\}}$, where $d$ is a semimetric of negative type.
\item The \emph{diffusion kernel} \citep{KonBar10} is given by

$\displaystyle{k_{\beta}(\sigma, \sigma^\prime) =  \exp\left\{\beta q(\sigma \circ\sigma^\prime)\right\}}$, where $\beta\in\mathbb{R}$ and $q$ is a function that must satisfy $q(\pi)=q(\pi^{-1})$ and $\sum_{\pi}q(\pi)=0$. A particular case is $q(\sigma,\sigma')=1$ if $\sigma$ and $\sigma'$ are connected by an edge in some Cayley graph representation of $S_n$, and  $q(\sigma,\sigma')=-\text{degree}_{\sigma}$ if $\sigma=\sigma'$ or $q(\sigma,\sigma')=0$ otherwise.
\item The \emph{semimetric or distance induced kernel} \citep{SejSriGreFuk13}, if the semimetric $d$ is of negative type, then, a family of  kernels $k$, parameterised by a central permutation $\sigma_0$, is given by

$\displaystyle k(\sigma,\sigma')= \frac{1}{2}\left[d(\sigma,\sigma_0)+d(\sigma',\sigma_0)-d(\sigma,\sigma')\right]$.
\end{enumerate}	
If we choose any of the above kernels by itself, it will generally not be complex enough to represent the ranking data's generating mechanism. However, we can benefit from the allowable operations for kernels to combine kernels and still obtain a valid kernel. Some of the operations which render a valid kernel are the following: sum, multiplication by a positive constant, product, polynomial and exponential  \citep{BerAgn04}.

In the case of the symmetric group of degree $n$, $S_n$, there exist kernels that are \emph{right invariant}, as defined in Equation~\eqref{rightindef}. This invariance property is useful because it is possible to write down the kernel as a function of a single argument and then obtain a Fourier representation. The caveat is that this Fourier representation is given in terms of certain matrix unitary representations due to the non-Abelian structure of the group \citep{Jam78}. Even though the space is finite, and every irreducible representation is finite-dimensional \citep{FukSriGreScho09}, these Fourier representations do not have closed form expressions. For this reason, it is difficult to work on the spectral domain as opposed to the $\mathbb{R}^n$ case. There  is also no natural measure to sample from such as the one provided by Bochner's theorem in Euclidean spaces \citep{Wen05}. In the next section, we will present a novel Monte Carlo kernel estimator for the case of partial rankings data.

\section{Partial rankings}

Having provided an overview of kernels for permutations, and reviewed the link between permutations and rankings of objects, we now turn to the practical issue that in real datasets, we typically have access only to partial ranking information, such as pairwise preferences and top-$k$ rankings. Following \cite{JiaoVert15}, we consider the following types of partial rankings:

\begin{definition}[Exhaustive partial rankings, top-$k$ rankings]
	Let $n \in \mathbb{N}$. A partial ranking on the set $[n]$ is specified by an ordered collection $\Omega_1 \succ \cdots \succ \Omega_l$ of disjoint non-empty subsets $\Omega_1,\ldots,\Omega_l \subseteq [n]$, for any $1 \leq l \leq n$. The partial ranking $\Omega_1 \succ \cdots \succ \Omega_l$ encodes the fact that the items in $\Omega_i$ are preferred to those in $\Omega_{i+1}$, for $i=1,\ldots,l-1$, with no preference information specified about the items in $[n] \setminus \cup_{i=1}^l \Omega_i$. A partial ranking $\Omega_1 \succ \cdots \succ \Omega_l$ with $\cup_{i=1}^l \Omega_i = [n]$ termed \emph{exhaustive}, as all items in $[n]$ are included within the preference information. A top-$k$ partial ranking is a particular type of exhaustive ranking $\Omega_1 \succ \cdots \succ \Omega_{l}$, with $|\Omega_1| = \cdots = |\Omega_{l-1}| = 1$, and $\Omega_{l} = [n] \setminus \cup_{i=1}^{l-1} \Omega_i$. We will frequently identify a partial ranking $\Omega_1 \succ \cdots \succ \Omega_l$ with the set $R(\Omega_1,\ldots,\Omega_l) \subseteq S_n$ of full rankings consistent with the partial ranking. Thus, $\sigma \in R(\Omega_1,\ldots,\Omega_l)$ iff for all $1\leq i< j \leq l$, and for all $x \in \Omega_i, y \in \Omega_j$, we have $\sigma^{-1}(x) < \sigma^{-1}(y)$. When there is potential for confusion, we will use the term ``subset partial ranking" when referring to a partial ranking as a subset of $S_n$, and ``preference partial ranking" when referring to a partial ranking with the notation $\Omega_1 \succ \cdots \succ \Omega_l$.
\end{definition}

Thus, for many practical problems, we require definitions of kernels between subsets of partial rankings rather than between full rankings, to be able to deal with datasets containing only partial ranking information. A common approach \citep{TsuKinAsa02} is to take a kernel $K$ defined on $S_n$, and use the \emph{marginalised kernel}, defined on subsets of partial rankings by
\begin{align}\label{convoker}
K(R,R')& =\sum_{\sigma \in R}\sum_{\sigma^\prime \in R^\prime} K(\sigma, \sigma^\prime)p(\sigma|R)p(\sigma^\prime|R^\prime) \,
\end{align}
for all $R, R^\prime \subseteq S_n$, for some probability distribution $p \in \mathscr{P}(S_n)$. Here, $p(\cdot |R)$ denotes the conditioning of $p$ to the set $R \subseteq S_n$. \citet{JiaoVert15} use the \emph{convolution kernel} \citep{Haussler99} between partial rankings, 

given by
\begin{align}\label{convo}
K(R, R^\prime) &= \frac{1}{|R| |R^\prime|} \sum_{\sigma \in R}\sum_{\sigma^\prime \in R^\prime} K(\sigma, \sigma^\prime).
\end{align}
This is a particular case for the marginalised kernel of Equation~\eqref{convoker}, in which we take the probability mass function to be uniform over $R,R'$ respectively. In general, computation with a marginalised kernel quickly becomes computationally intractable, with the number of terms in the right-hand side of Equation \eqref{convoker} growing super-exponentially with $n$, for a fixed number of items in the partial rankings $R$ and $R^\prime$, see Appendix~\ref{bignumbers} for a numerical example of such growth.
An exception is the Kendall kernel case for two interleaving partial rankings of $k$ and $m$ items or a top-$k$ and top-$m$ ranking. In this case, the sum can be tractably computed and it can be done in $\mathcal{O}(k \log k + m \log m)$ time  \citep{JiaoVert15}.

We propose a variety of Monte Carlo methods to estimate the marginalised kernel of Equation~\eqref{convoker} for the general case, where direct calculation is intractable. 

\begin{definition}
The Monte Carlo estimator approximating the marginalised kernel of Equation~\eqref{convoker} is defined for a collection of partial rankings $(R_i)_{i=1}^I$, given by
\begin{align}\label{monteker}
\widehat{K}(R_i, R_j) & = \frac{1}{M_i M_j}\sum_{l=1}^{M_i} \sum_{m=1}^{M_j} w^{(i)}_l w^{(j)}_m K(\sigma^{(i)}_l, \sigma^{(j)}_m)
\end{align}
for $i,j = 1,\ldots,I$, where $((\sigma^{(i)}_n)_{m=1}^{M_i})_{i=1}^I$ are random permutations, and $\left((w^{(i)}_m)_{m=1}^{M_i}\right)_{i=1}^I$ are random weights. Note that this general set-up allows for several possibilities:
\begin{itemize}
	\item For each $i=1\ldots,I$, the permutations $(\sigma^{(i)}_m)_{m=1}^{M_i}$ are drawn exactly from the distribution $p(\cdot|R_i)$. In this case, the weights are simply $w^{(i)}_n = 1$ for $m=1,\ldots,M_i$.
	\item For each $i=1,\ldots,I$, the permutations $(\sigma^{(i)}_m)_{m=1}^{M_i}$ drawn from some proposal distribution $q(\cdot|R_i)$  with the weights given by the corresponding \emph{importance weights} $w^{(i)}_n = p(\sigma^{(i)}_n|R) / q(\sigma^{(i)}_n|R)$ for $m=1,\ldots,M_i$.

\end{itemize}
\end{definition}
An alternative perspective on the estimator defined in Equation \eqref{monteker}, more in line with the literature on random feature approximations of kernels, is to define a random feature embedding for each of the partial rankings $(R_i)_{i=1}^I$.

More precisely, let $\mathcal{H}_K$ be the (finite-dimensional) Hilbert space associated with the kernel $K$ on the space $S_n$, and let $\boldsymbol{\Phi}$ be the associated feature map, so that $\Phi(\sigma) = K(\sigma, \cdot) \in \mathcal{H}_K$ for each $\sigma \in S_n$. Then observe that we have $K(\sigma, \sigma^\prime) = \langle \boldsymbol{\Phi}(\sigma), \boldsymbol{\Phi}(\sigma^\prime) \rangle$ for all $\sigma, \sigma^\prime \in S_n$. We now extend this feature embedding to partial rankings as follows. Given a partial ranking $R \subseteq S_n$, we define the feature embedding of $R$ by
\[
\boldsymbol{\Phi}(R) = \frac{1}{|R|} \sum_{\sigma \in R} K(\sigma, \cdot) \in \mathcal{H}_K
\]
With this extension of $\boldsymbol{\Phi}$ to partial rankings, we may now directly express the marginalised kernel of Equation \eqref{convoker} as an inner product in the same Hilbert space $\mathcal{H}_K$:
\[
K(R, R^\prime) = \langle \boldsymbol{\Phi}(R), \boldsymbol{\Phi}(R^\prime) \rangle
\]
for all partial rankings $R, R^\prime \subseteq S_n$. If we define a random feature embedding of the partial rankings $(R_i)_{i=1}^I$ by
\[
\widehat{\boldsymbol{\Phi}}(R_i) = \sum_{m=1}^{M_i} w^{(i)}_m \boldsymbol{\Phi}(\sigma^{(i)}_m)
\]
then the Monte Carlo kernel estimator of Equation \eqref{monteker} can be expressed directly as
\begin{align}
\widehat{K}(R_i, R_j) & = \frac{1}{M_i M_j}\sum_{l=1}^{M_i} \sum_{m=1}^{M_j} w^{(i)}_l w^{(j)}_m K(\sigma^{(i)}_l, \sigma^{(j)}_m) \nonumber \\
& = \frac{1}{M_i M_j}\sum_{l=1}^{M_i} \sum_{m=1}^{M_j} w^{(i)}_l w^{(j)}_m \langle \boldsymbol{\Phi}(\sigma_l^{(i)}), \boldsymbol{\Phi}(\sigma_m^{(j)}) \rangle \nonumber\\
& = \left\langle \frac{1}{M_i}\sum_{l=1}^{M_i} w_l^{(i)} \boldsymbol{\Phi}(\sigma_l^{(i)}), \frac{1}{M_j}\sum_{m=1}^{M_j} w_m^{(j)} \boldsymbol{\Phi}(\sigma_m^{(j)}) \right\rangle \nonumber \\
& = \langle \widehat{\boldsymbol{\Phi}}(R_i) , \widehat{\boldsymbol{\Phi}}(R_j)\rangle \label{eq:mckernelasinnerprod}
\end{align}
for each $i,j \in \{1,\ldots, I\}$. This expression of the estimator as an inner product between randomised embeddings will be useful in the sequel.

We provide an illustration of the various RKHS embeddings at play in Figure \ref{fig:embeddings}, using the notation of the proof of Theorem \ref{thm:kernel-and-unbiased}. In this figure, $\eta$ is a partial ranking, with three consistent full rankings $\sigma_1, \sigma_2, \sigma_3$. The extended embedding $\widetilde{\boldsymbol{\Phi}}$ applied to $\eta$ is the barycentre in the RKHS of the embeddings of the consistent full rankings, and a Monte Carlo approximation $\widehat{\boldsymbol{\Phi}}$ to this embedding is also displayed.

\begin{figure}[h]
	\centering
	\includegraphics[keepaspectratio, width=0.5\textwidth]{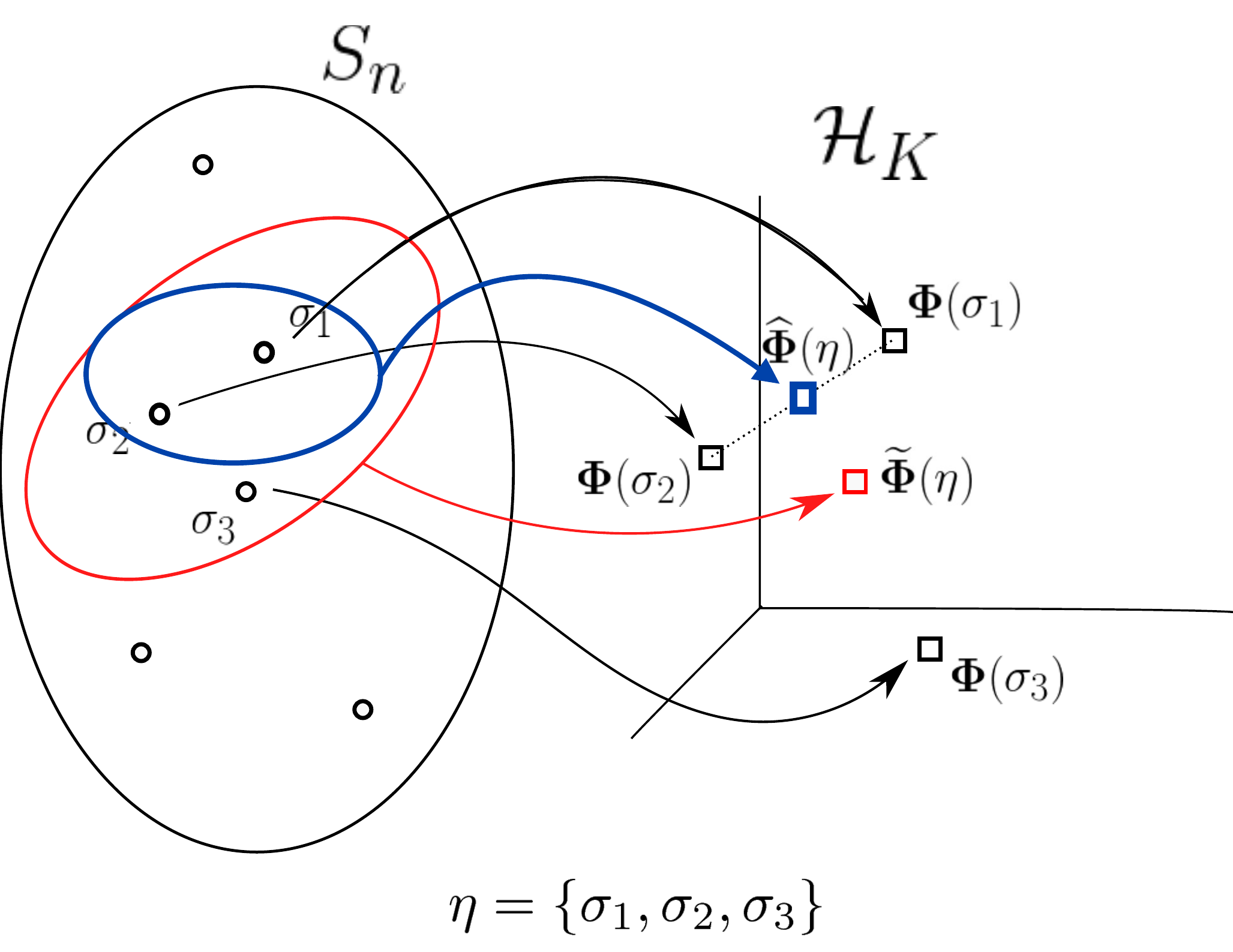}
	\caption{Visualisation of the various embeddings discussed in the proof of Theorem \ref{thm:kernel-and-unbiased}. $\sigma_1, \sigma_2, \sigma_3$ are permutations in $S_n$, which are mapped into the RKHS $\mathcal{H}_K$ by the embedding $\boldsymbol{\Phi}$. $\eta$ is a partial ranking subset  which contains $\sigma_1, \sigma_2, \sigma_3$, and its embedding $\widetilde{\boldsymbol{\Phi}}(\eta)$ is given as the average of the embeddings of its full rankings. The Monte Carlo embedding $\widehat{\boldsymbol{\Phi}}(\eta)$ induced by Equation \eqref{monteker} is computed by taking the average of a randomly sampled collection of consistent full rankings from $\eta$.}

	\label{fig:embeddings}
\end{figure}

\begin{theorem}\label{thm:embeddingtheorem} Let $R_i\subseteq S_n$ be a partial ranking, and let $\left(\sigma_m^{(i)}\right)_{m=1}^{M_i}$ independent and identically distributed samples from $p(\cdot \mid R_i)$. The kernel Monte Carlo mean embedding,
\[
\widehat{\Phi}(R_i)=\frac{1}{M_i}\sum_{m=1}^{M_i} K(\sigma_m^{(i)},\cdot)
\]
is a consistent estimator of the marginalised kernel embedding
\[
\widetilde{\Phi}(R_i) = \frac{1}{|R_i|}\sum_{\sigma \in R_i} K(\sigma, \cdot) \, .
\]
\end{theorem}

\begin{proof}
Note that the RKHS in which these embeddings take values is finite-dimensional, and the Monte Carlo estimator is the average of iid terms, each of which is equal to the true embedding in expectation. Thus, we immediately obtain unbiasedness and consistency of the Monte Carlo embedding.

\end{proof}

\begin{theorem}\label{thm:kernel-and-unbiased}
	The Monte Carlo kernel estimator from Equation~\eqref{monteker} does define a positive-definite kernel; further, it yields consistent estimates of the true kernel function.
\end{theorem}
\begin{proof}
	We first deal with the positive-definiteness claim. Let $R_1, \ldots, R_I \subseteq S_n$ be a collection of partial rankings, and for each $i=1,\ldots,I$, let $(\sigma^{(i)}_{m}, w^{(i)}_{m})_{m=1}^{M_i}$ be an i.i.d. weighted collection of complete rankings distributed according to $p(\cdot | R_i)$. To show that the Monte Carlo kernel estimator $\widehat{K}$ is positive-definite, we observe that by Equation \eqref{eq:mckernelasinnerprod}, the $I \times I$ matrix with $(i,j)$\textsuperscript{th} element given by $\widehat{K}(R_i, R_j)$ is the Gram matrix of the vectors $(\widehat{\boldsymbol{\Phi}}(R_i))_{i=1}^I$ with respect to the inner product of the Hilbert space $\mathcal{H}_K$. We therefore immediately deduce that the matrix is positive semi-definite, and therefore the kernel estimator itself is positive-definite. Furthermore, the Monte Carlo kernel estimator is consistent; see Appendix~\ref{app:unbiasednessproof} in the supplementary material for the proof.
\end{proof}

Having established that the Monte Carlo estimator $\widehat{K}$ is itself a kernel, we note that when it is evaluated at two partial rankings $R, R^\prime \subseteq S_n$, the resulting expression is \textit{not} a sum of iid terms; the following result quantifies the quality of the estimator through its variance.

\begin{theorem}\label{varmonteker}
	The variance of the Monte Carlo kernel estimator evaluated at a pair of partial rankings $R_i, R_j$, with $M_i, N_j$ Monte Carlo samples respectively, is given by
	\begin{align*}
	&\mathrm{Var}\left(\widehat{K}(R_{i},R_{j})\right) = \\
	&\frac{1}{M_i}\sum_{\sigma^{(i)}\in R_i}p(\sigma^{(i)}| R_i)\left(\sum_{\sigma^{(j)}\in R_j}p(\sigma^{(j)}| R_j)K(\sigma^{(i)},\sigma^{(j)})\right)^2\\
	&\frac{-1}{M_i}\biggl(\sum_{\substack{\sigma^{(i)}\in R_i \\ \sigma^{(j)}\in R_j }}K(\sigma^{(i)},\sigma^{(j)})p(\sigma^{(i)}| R_i)p(\sigma^{(j)}| R_j)\biggr) ^2\\
	&\frac{-1}{M_iN_j}\sum_{\sigma^{(i)}\in R_i}p(\sigma^{(i)}| R_i)\biggl(\sum_{\sigma^{(j)}\in R_j}p(\sigma^{(j)}| R_j)K(\sigma^{(i)},\sigma^{(j)})\biggr)^2\\
	&+\frac{1}{M_iN_j}\sum_{\substack{\sigma^{(i)}\in R_i \\ \sigma^{(j)}\in R_j }}K(\sigma^{(i)},\sigma^{(j)})^2p(\sigma^{(i)}| R_i)p(\sigma^{(j)}| R_j).
	\end{align*}
\end{theorem}

The proof is given in the supplementary material, Appendix~\ref{app:varMCK}. We have presented some theoretical properties of the embedding corresponding to the Monte Carlo kernel estimator which confirm that it is a sensible embedding. In the next section, we present a lower variance estimator based on a novel antithetic variates construction.

\section{Antithetic random variates for permutations}

A common, computationally cheap variance reduction technique in Monte Carlo estimation of expectations of a given function is to use antithetic variates \citep{HamMor56}, the purpose of which is to introduce negative correlation between samples without affecting their marginal distribution, resulting in a lower variance estimator. Antithetic samples have been used when sampling from Euclidean vector spaces, for which antithetic samples are straightforward to define. However, to the best of our knowledge, antithetic variate constructions have not been proposed for the space of permutations. We begin by introducing a definition for antithetic samples for permutations.

\begin{definition}[Antithetic permutations]\label{def:antithetic}
	Let $R \subseteq S_n$ be a top-$k$ partial ranking. The antithetic operator $A_{R} : R \rightarrow R$ maps each permutation $\sigma \in R$ to the permutation in $R$ of maximal distance from $\sigma$.
\end{definition}
It is not necessarily clear a priori that the antithetic operator of Definition \ref{def:antithetic} is well-defined, but for the Kendall distance and top-$k$ partial rankings, it turns out that it is indeed well-defined.

\begin{rmk}\label{rem:topkantithetic}
	For the Kendall distance and top-$k$ partial rankings, the antithetic operators of Definition \ref{def:antithetic} are well-defined, in the sense that there exists a unique distance-maximising permutation in $R$ from any given $\sigma \in R$.
	Indeed, the antithetic map $A_R$ when $R$ is a top-$k$ partial ranking has a particularly neat expression; if the partial ranking corresponding to $R$ is $a_1 \succ \cdots \succ a_k$, and we have a full ranking $\sigma \in R$ (so that $\sigma(1) = a_1,\ldots,\sigma(k) = a_k$, then the antithetic permutation $A_R(\sigma)$ is given by
	\begin{align}
		A_R(\sigma)(i) =& a_i \ \ &\text{ for } i=1,\ldots,k \, ,\nonumber\\
		A_R(\sigma)(k + j) =& \sigma(n+1-j)  \ \ &\text{ for } j=1,\ldots,n-k \nonumber\, .
	\end{align}
	In this case, we have $d(\sigma, A_R(\sigma)) = \binom{n-k}{2}$.
\end{rmk}

This definition of antithetic samples for permutations has parallels with the standard notion of antithetic samples in vector spaces, in which typically a sampled vector $x \in \mathbb{R}^d$ is negated to form $-x$, its antithetic sample; $-x$ is the vector maximising the Euclidean distance from $x$, under the restrictions of fixed norm. 

\begin{prp}\label{prop:negcov} Let $R$ be a partial ranking and  $\left\{ \sigma,\anti{\sigma}{R}\right\}$ be an antithetic pair from $R$,  $\sigma$ distributed Uniformly in the region $R$. Let $d:S_n\rightarrow \mathbb{R}^{+}$ be the Kendall distance and $\sigma_0\in R$ a fixed permutation, then $X=d(\sigma,\sigma_0)$ and $ Y= d(\anti{\sigma}{R},\sigma_0)$, then, $X$ and $Y$ have negative covariance.
\end{prp}
The proof of this proposition is presented after the relevant lemmas are proved. Since one of the main tasks in statistical inference is to compute expectations of a function of interest, denoted by $h$, once the antithetic variates are constructed, the functional form of $h$ determines whether or not the antithetic variate construction produces a lower variance estimator for its expectation.
If $h$ is a monotone function, we have the following corollary.

\begin{cor}\label{cor:hfunc} Let $h$ be a monotone increasing (decreasing) function. Then, the random variables $h\left( X\right)$ and $h\left(Y\right)$, have negative covariance.
\end{cor}
 \begin{proof} The random variable $Y$ from Proposition~\ref{prop:negcov} is equal in distribution to $Y\stackrel{d}{=}K-X$, where $K$ is a constant which changes depending whether $\sigma$ is a full ranking or an exhaustive partial ranking, see the proof of Proposition~\ref{prop:negcov} in the next section for the specific form of the constants. By Chebyshev's integral inequality~\citep{FinJod84}, the covariance between a monotone increasing (decreasing) and a monotone decreasing (increasing) functions is negative.
\end{proof}

The next theorem presents the antithetic empirical feature embedding and corresponding antithetic kernel estimator. Indeed, if we take the inner product between two embeddings, this yields the kernel antithetic estimator which is a function of a pair of partial rankings subsets. In this case, the $h$ function from above is the kernel evaluated in each pair, this is an example of a $U$-statistic \citep[Chapter 5]{Serf80}. 
\begin{theorem}\label{anti} Let $R_i\subseteq S_n$ be a partial ranking, $S_n$ denotes the space of permutations of $n\in \mathbb{N}$, $(\sigma_m^{(i)},\anti{\sigma_m^{(i)}}{R_i})_{m=1}^{M_i}$ are antithetic pairs of i.i.d. samples from the region $R_i$. The Kernel antithetic Monte Carlo mean embedding is

\[
\widehat{\phi}(R_i)=\frac{1}{M_i}\sum_{m=1}^{M_i}\left[\frac{k(\sigma_m^{(i)},\cdot)+k(\anti{\sigma_m^{(i)}}{R_i},\cdot)}{2}\right].
\]
It is a consistent estimator of the embedding that corresponds to the marginalised kernel
\begin{gather}\label{antikerest}
\frac{1}{4NM}\sum_{n=1}^N \sum_{m=1}^M \bigl( K(\sigma_n, \tau_m)\nonumber\\ + K(\widetilde{\sigma}_n, \tau_m) +K(\sigma_n, \widetilde{\tau}_m) + K(\widetilde{\sigma}_n, \widetilde{\tau}_m) \bigr)
\end{gather}
\end{theorem}

\begin{proof}
Since the estimator is a convex combination of the Monte Carlo Kernel estimator, consistency follows.
\end{proof}

 In the next section, we present the main result about the estimator from Theorem~\ref{anti}, namely, that it has lower asymptotic variance than the Monte Carlo kernel estimator from Equation~\eqref{monteker}.

\subsection{Variance of the antithetic kernel estimator}
We now establish some basic theoretical properties of antithetic samples in the context of marginalised kernel estimation. In order to do so, we require a series of lemmas to derive the main result in Theorem~\ref{varantiker} that guarantees that the antithetic kernel estimator has lower asymptotic variance than the Monte Carlo kernel estimator for the marginalised Mallows kernel.

 The following result shows that antithetic permutations may be used to achieve coupled samples which are marginally distributed uniformly on the subset of $S_n$ corresponding to a top-$k$ partial ranking.

\begin{lemma}\label{lem:antithetic-marginal}
	If $R \subseteq S_n$ is a top-$k$ partial ranking, then if $\sigma \sim \textrm{Unif}(R)$, then $A_{R}(\sigma) \sim \textrm{Unif}(R)$.
\end{lemma}
\begin{proof}
	The proof is immediate from Remark \ref{rem:topkantithetic}, since $A_R$ is bijective on $R$.
\end{proof}

Lemma \ref{lem:antithetic-marginal} establishes a base requirement of an antithetic sample -- namely, that it has the correct marginal distribution. In the context of antithetic sampling in Euclidean spaces, this property is often trivial to establish, but the discrete geometry of $S_n$ makes this property less obvious. Indeed, we next demonstrate that the condition of exhaustiveness of the partial ranking in Lemma \ref{lem:antithetic-marginal} is neccessary.

\begin{example}\label{ex:needexhaustive} 
	Let $n=3$, and consider the partial ranking $2 \succ 1$. Note that this is not an exhaustive partial ranking, as the element $3$ does not feature in the preference information. There are three full rankings consistent with this partial ranking, namely $3 \succ 2 \succ 1$, $2 \succ 3 \succ 1$, and $2 \succ 1 \succ 3$. Encoding these full rankings as permutations, as described in the correspondence outlined in Section \ref{sec:def}, we obtain three permutations, which we respectively denote by $\sigma_A, \sigma_B, \sigma_C \in S_3$. Specifically, we have
	\begin{align*}
	\sigma_A(1) = 3 \, , \ \ \sigma_A(2) = 2\, , \ \ \sigma_A(3) = 1 \, . \\
	\sigma_B(1) = 2 \, , \ \ \sigma_B(2) = 3\, , \ \ \sigma_A(3) = 1 \, . \\
	\sigma_C(1) = 2 \, , \ \ \sigma_C(2) = 1\, , \ \ \sigma_A(3) = 3 \, .
	\end{align*}
	Under the right-invariant Kendall distance, we obtain pairwise distances given by
	\begin{align*}
	d(\sigma_A, \sigma_B) = 1 \, , \\
	d(\sigma_A, \sigma_C) = 2 \, , \\ 
	d(\sigma_B, \sigma_C) = 1 \, .
	\end{align*}
	Thus, the marginal distribution of an antithetic sample for the partial ranking $2 \succ 1$ places no mass on $\sigma_B$, and half of its mass on each of $\sigma_A$ and $\sigma_C$, and is therefore not uniform over $R$.
\end{example}

We further show that the condition of right-invariance of the metric $d$ is necessary in the next example.

\begin{example}\label{ex:needrightinv}
	Let $n=3$, and suppose $d$ is a distance on $S_3$ such that, with the notation introduced in Example \ref{ex:needexhaustive}, we have
	\begin{align*}
	d(\sigma_A, \sigma_B) = 1 \, , \\
	d(\sigma_A, \sigma_C) = 0.5 \, , \\ 
	d(\sigma_B, \sigma_C) = 1 \, .
	\end{align*}
	Note that $d$ is not right-invariant, since
	\begin{align*}
	&d((\sigma_A, \sigma_C) \\
	=& d(\sigma_B\tau, \sigma_A\tau) \\
	\not=& d(\sigma_B, \sigma_A)  \, ,
	\end{align*}
	where $\tau \in S_3$ is given by $\tau(1) = 1, \tau(2) = 3, \tau(3) =2$. 
	Then note that an antithetic sample for the kernel associated with this distance and the partial ranking $1 \succ 2$, is equal to $\sigma_B$ with probability $2/3$ and the other two full rankings with probability $1/6$ each, and therefore does not have a uniform distribution.
\end{example}

Examples~\ref{ex:needexhaustive} and~\ref{ex:needrightinv} serve to illustrate the complexity of antithetic sampling constructions in discrete spaces.

 The following two lemmas state some useful relationships between the distance between two permutations $(\sigma,\tau)$ and the corresponding pair $(\anti{\sigma}{R},\tau)$ in both the unconstrained and constrained cases which correspond to not having any partial ranking information and having partial ranking information, respectively. \begin{lemma}\label{lem:antitheticfull}
	Let $\sigma, \tau \in S_n$. Then, $d(\sigma, \tau) = $
	
	\noindent $\binom{n}{2} - d(A_{S_n}(\sigma), \tau)$.
\end{lemma}
\begin{proof}
	This is immediate from the interpretation of the Kendall distance as the number of discordant pairs between two permutations; a distinct pair $i,j \in [n]$ are discordant for $\sigma, \tau$ iff they are concordant for $A_{S_n}(\sigma), \tau$.
\end{proof}
In fact, Lemma~\ref{lem:antitheticfull} generalises in the following manner.

\begin{lemma}\label{lem:antitheticdistancepartial}
	Let $R$ be a top-$k$ ranking $a_1 \succ \cdots \succ a_l \succ [n] \setminus \{a_1, \ldots, a_l\}$, and 
	let $\sigma, \tau \in R$. Then $d(\sigma, \tau) = \binom{n-l}{2} - d(A_{R}(\sigma), \tau)$.
\end{lemma}
\begin{proof}
	As for the proof of Lemma \ref{lem:antitheticfull}, we use the ``discordant pairs'' interpretation of the Kendall distance. Note that if a distinct pair $\{ x, y \} \in [n]^{(2)}$ has at least one of $x, y \in \{a_1,\ldots,a_l\}$, then by virtue of the fact that $\sigma, A_R(\sigma), \tau \in R$, any pair of these permutations is concordant for $x,y$. Now observe that any distinct pair $x,y \in [n] \setminus \{a_1, \ldots, a_l\}$ is discordant for $\sigma, \tau$ iff it is concordant for $A_R(\sigma), \tau$, from the construction of $A_R(\sigma)$ described in Remark \ref{rem:topkantithetic}. The total number of such pairs is $\binom{n-l}{2}$, so we have $d(\sigma, \tau) + d(A_R(\sigma), \tau) = \binom{n-l}{2}$, as required.
\end{proof}

Next, we show that it is possible to obtain a unique closest element in a given partial ranking set $R$, denoted by $\Pi_R(\tau)$, with respect to any given permutation $\tau\in S_n,\tau\notin R$. This is based on the usual generalisation of a distance between a set and a point \citep{Dud02}. We then use such closest element in Lemmas~\ref{lem:decomp} and~\ref{lem:partialrelationanti} to obtain useful decompositions of distances identities. Finally, in Lemma~\ref{lem:unifOnSmallerRanking} we verify that the closest element is also distributed uniformly on a subset of the original set $R$.
\begin{lemma}\label{lem:closestpartialranking}
	Let $R \subseteq S_n$ be a top-$k$ partial ranking, let $\tau \in S_n$ be arbitrary. There is a unique closest element in $R$ to $\tau$. In other words, $\arg \min_{\sigma \in R}d(\sigma, \tau)$ 
	is a set of size 1.
\end{lemma}
\begin{proof}
	We use the interpretation of the Kendall distance as the number of discordant pairs between two permutations. Let $R$ be the top-$k$ partial ranking given by $x_1 \succ \cdots \succ x_k \succ [n] \setminus \{x_1,\ldots,x_k\}$, and let $X = \{x_1,\ldots, x_k\}$. We decompose the Kendall distance between $\sigma \in R$ and $\tau$ as follows:
	\begin{align}
	d(\sigma, \tau) = \sum_{x,y \in X, x \not= y} \mathbbm{1}_{x,y \text{ discordant for } \sigma, \tau} \nonumber\\ + \sum_{x \in X, y \not\in X} \mathbbm{1}_{x,y \text{ discordant for } \sigma, \tau}\nonumber\\ + \sum_{x,y \not\in X, x \not= y} \mathbbm{1}_{x,y \text{ discordant for } \sigma, \tau}  \, .
	\end{align}
	As $\sigma$ varies in $R$, only some of these terms vary. In particular, it is only the third term that varies with $\sigma$, and it is minimised at $0$ by the permutation $\sigma$ in $R$ which is in accordance with $\tau$ on the set $[n] \setminus X$.
\end{proof}
\begin{definition}
	Let $R \subseteq S_n$ be a top-$k$ partial ranking. Let $\Pi_R : S_n \rightarrow R$ be the map that takes a permutation to the corresponding Kendall-closest permutation in $R$; by Lemma \ref{lem:closestpartialranking}, this is well-defined.
\end{definition}

\begin{lemma}[Decomposition of distances]\label{lem:decomp}
	Let $\sigma \in R$, and $\tau \in S_n$. We have the following decomposition of the distance $d(\sigma, \tau)$:
	\begin{align*}
		d(\sigma, \tau) = d(\sigma, \Pi_R(\tau)) + d(\Pi_R(\tau), \tau) \, .
	\end{align*}
\end{lemma}
\begin{proof}
	We compute directly with the discordant pairs definition of the Kendall distance. Again, let $R$ be the partial ranking $x_1 \succ \cdots \succ x_k$, and let $X = \{x_1,\ldots, x_k\}$. We decompose the Kendall distance between $\sigma \in R$ and $\tau$ as before:
	\begin{align}\label{eq:decomp_sigma_tau}
	d(\sigma, \tau) = \sum_{x,y \in X, x \not= y} \mathbbm{1}_{x,y \text{ discordant for } \sigma, \tau} \nonumber\\+ \sum_{x \in X, y \not\in X} \mathbbm{1}_{x,y \text{ discordant for } \sigma, \tau} \nonumber\\+ \sum_{x,y \not\in X, x \not= y} \mathbbm{1}_{x,y \text{ discordant for } \sigma, \tau}  \, .
	\end{align}
	By the construction of $\Pi_R(\tau)$ in the proof of Lemma \ref{lem:closestpartialranking}, we have that
	\begin{align*}
	d(\Pi_R(\tau), \tau) = \sum_{x,y \in X, x \not= y} \mathbbm{1}_{x,y \text{ discordant for } \sigma, \tau} \nonumber\\+ \sum_{x \in X, y \not\in X} \mathbbm{1}_{x,y \text{ discordant for } \sigma, \tau} \, ,
	\end{align*}
	i.e. the first two terms of the decomposition in Equation \eqref{eq:decomp_sigma_tau}. Similarly, we have
	\begin{align*}
	d(\Pi_R(\tau), \sigma) = \sum_{x,y \not\in X, x \not= y} \mathbbm{1}_{x,y \text{ discordant for } \sigma, \tau} \, ,
	\end{align*}
	and so the result follows.
\end{proof}
\begin{lemma}{\label{lem:partialrelationanti}}
	Let $\sigma \in R$, and let $\tau \in R^\prime$. We have the following relationship between $d(A_{R}(\sigma), \tau)$ and $d(\sigma, \tau)$:
	\begin{align}
	d(A_{R}(\sigma), \tau) = d(\sigma, \tau) + \binom{n-k}{2} - 2d(\sigma, \Pi_R(\tau))\, .
	\end{align}
\end{lemma}
\begin{proof}
	We begin by observing that, by Lemma \ref{lem:decomp}, we have
	\begin{align}
	d(\sigma, \tau) = d(\sigma, \Pi_R(\tau)) + d(\Pi_R(\tau), \tau) \, ,
	\end{align}
	and 
	\begin{align}
	d(A_{R}(\sigma), \tau) = d(A_R(\sigma), \Pi_R(\tau)) + d(\Pi_R(\tau), \tau) \, .
	\end{align}
	Now, from Lemma \ref{lem:antitheticdistancepartial}, we have that $d(A_R(\sigma), \Pi_R(\tau)) = \binom{n-k}{2} - d(\sigma, \Pi_R(\tau))$. Hence, the result follows.
\end{proof}

\begin{lemma}\label{lem:unifOnSmallerRanking}
	Let $R, R^\prime \subseteq S_n$ be top-$k$ rankings, in preference notation given by
	\begin{align}
	R:& a_1 \succ \cdots \succ a_l \succ [n] \setminus \{a_1,\ldots,a_l\} \, , \nonumber \\
	R^\prime:& b_1 \succ \cdots \succ b_m \succ [n] \setminus \{b_1,\ldots,b_m\} \nonumber  \, .
	\end{align}
	If $\tau \sim \text{Unif}(R^\prime)$, then $\Pi_R(\tau)$ is a full ranking with distribution $\text{Unif}(R^{\prime\prime})$, where $R^{\prime\prime} \subseteq R$ is the partial ranking given by
	\begin{align}
	R^{\prime\prime}:  a_1 \succ \cdots \succ a_l \succ b_{i_1} \succ \cdots \succ b_{i_q} \nonumber \\ \succ [n] \setminus \{a_1,\ldots,a_l, b_1,\ldots, b_m\} \nonumber  \, ,
	\end{align}
	where $\{b_{i_1},\ldots,b_{i_q}\} = \{b_1,\ldots,b_m\} \setminus \{a_1,\ldots, a_l\}$, and $i_{j} < i_{j+1}$ for all $j=1,\ldots,q-1$.
\end{lemma}
\begin{proof}
	We first show that $\Pi_R$ maps $R^\prime$ into $R^{\prime\prime}$. This is straightforward, as given $\tau \in R^\prime$, we first observe that $\Pi_R(\tau) \in R$, and so the full ranking $\Pi_R(\tau)$ is consistent with the partial ranking
	\begin{align}
		a_1 \succ \cdots \succ a_l \succ [n] \setminus \{a_1,\ldots,a_l\} \nonumber  \, .
	\end{align}
	Next, since $\Pi_R(\tau)$ is concordant with $\tau$ for all pairs outside the set $\{a_1,\ldots, a_l\}$, $\Pi_R(\tau)$ must be consistent with the partial ranking
	\begin{align}
	b_{i_1} \succ \cdots \succ b_{i_q} \succ [n] \setminus \{a_1,\ldots,a_l, b_1,\ldots, b_m\}  \nonumber \, .
	\end{align}
	Putting these two facts together shows that the full ranking $\Pi_R(\tau)$ must be consistent with the partial ranking
	\begin{align}
	a_1 \succ \cdots \succ a_l \succ b_{i_1} \succ \cdots \succ b_{i_q} \nonumber \\ \succ [n] \setminus \{a_1,\ldots,a_l, b_1,\ldots, b_m\}  \nonumber \, .
	\end{align}
	Thus, given $\tau \sim \text{Unif}(R^\prime)$, the distribution of $\Pi_R(\tau)$ is supported on $R^{\prime\prime}$. To show that it is uniform, we now argue that equally many rankings in $R^\prime$ are mapped to each ranking in $R^{\prime\prime}$. To see this, we observe that the pre-image of a ranking in $R^{\prime\prime}$ is the set of all rankings in $R^{\prime}$ which are concordant with it on all pairs in $[n] \setminus \{a_1,\ldots,a_l, b_1,\ldots, b_m\}$. The number of such rankings is independent of the selected ranking in $R^{\prime\prime}$, and so the statement of the lemma follows.
\end{proof}

Having introduced the antithetic operator for a top-$k$ partial ranking $R$, $A_R : R \rightarrow R$ and the projection map $\Pi_R : S_n \rightarrow R$, we next study how these operations interact with one another.
\begin{lemma}
	Let $R^{\prime\prime} \subseteq R \subseteq S_n$ be top-$k$ partial rankings. Then for $\sigma \in R$, we have
	\begin{align}
	A_{R^{\prime\prime}}(\Pi_{R^{\prime\prime}}(\sigma)) = \Pi_{R^{\prime\prime}}( A_{R}(\sigma) )  \nonumber \, .
	\end{align}
\end{lemma}
\begin{proof}
	We begin by introducing preference-style notation for $R$ and $R^{\prime\prime}$. Let $R$ be the top-$k$ ranking given by $a_1 \succ \cdots \succ a_l \succ [n] \setminus \{a_1,\ldots,a_l\}$, and let $R^{\prime\prime}$ be the partial ranking given by $a_1 \succ \cdots \succ a_l \succ a_{l+1} \succ \cdots \succ a_m \succ [n] \setminus \{a_1,\ldots,a_m\}$. Let $\sigma \in R$, and let the elements of $[n] \setminus \{a_1,\ldots,a_m\}$ be given by $b_1,\ldots,b_q$, with indices chosen such that $\sigma$ corresponds to the full ranking
	\begin{align}
	a_1 \succ \cdots a_m \succ b_1 \succ \cdots \succ b_q  \nonumber \, .
	\end{align}
	Then, the ranking $A_{R^{\prime\prime}}(\Pi_{R^{\prime\prime}}(\sigma))$ is given by
	\begin{align}
	a_1 \succ \cdots a_m \succ b_q \succ \cdots \succ b_1  \nonumber \, ,
	\end{align}
	and a straightforward calculation shows that this is also the case for $\Pi_{R^{\prime\prime}}(A_R(\sigma))$, as required.
\end{proof}
Finally, the last Lemma states the most general identity for a distance, which involves the antithetic operator, the closest element map given a partial rankings set $R$ and a subset of it, denoted by $R''$.
\begin{lemma}\label{lem:projAndAnti}
	Let $R^{\prime\prime} \subseteq R \subseteq S_n$ be top-$k$ partial rankings, given in preference notation by
	\begin{align}
	&R: a_1 \succ \cdots \succ a_l \succ [n] \setminus \{a_1, \ldots, a_l\} \, , \nonumber \\
	&R^{\prime\prime}: a_1 \succ \cdots \succ a_l \succ a_{l+1} \succ \cdots a_m \succ [n] \setminus \{a_1, \ldots, a_m\} \nonumber  \, .
	\end{align}
	Let $\alpha$ be the number of unranked elements under $R$, and let $\beta$ be the additional number of elements ranked under $R^{\prime\prime}$ relative to $R$. Then for $\sigma \in R$, we have
	\begin{align}
		d(\sigma, \Pi_{R^{\prime\prime}}(\sigma)) = ((n-l) - (m-l))(m-l) \nonumber \\ + \binom{m-l}{2} -  d(A_{R}(\sigma), \Pi_{R^{\prime\prime}}(A_R(\sigma))) \nonumber  \, .
	\end{align}
\end{lemma}
\begin{proof}
	Again, we denote $\{b_1, \ldots, b_q\} = [n] \setminus \{a_1,\ldots,a_m\}$, with indices chosen such that $\sigma$ corresponds to the full ranking $
	a_1 \succ \cdots \succ a_m \succ b_1 \succ \cdots \succ b_q$. 
	From earlier arguments, we have
	\begin{align}
	d(\sigma, \Pi_{R^{\prime\prime}}(\sigma)) = \sum_{\substack{x \in \{a_{l+1}, \cdots, a_m\} \\ y \in \{a_{l+1},\ldots, a_m\}} } \mathbbm{1}_{(x,y)\text{ discordant for }  \sigma, \Pi_{R^{\prime\prime}}(\sigma)}  \nonumber \\+  \sum_{\substack{x \in \{a_{l+1}, \cdots, a_m\} \\ y \in \{b_1,\ldots, b_q\}} } \mathbbm{1}_{(x,y)\text{ discordant for }  \sigma, \Pi_{R^{\prime\prime}}(\sigma)}  \nonumber \, .
	\end{align}
	Now observe that for $a_i, a_j$ with $l+1 \leq i < j \leq m$, this pair is discordant for the pair of rankings $\sigma, \Pi_{R^{\prime\prime}}(\sigma)$ iff $a_j \succ a_i$ under $\sigma$ iff $a_i \succ a_j$ w.r.t $A_R(\sigma)$ iff $a_i, a_j$ are concordant for the pair of rankings $A_R(\sigma), \Pi_{R^{\prime\prime}}(A_R(\sigma))$. Hence, we have
	\begin{align}
	 \sum_{\substack{x \in \{a_{l+1}, \cdots, a_m\} \\ y \in \{a_{l+1},\ldots, a_m\}} } \mathbbm{1}_{(x,y)\text{ discordant for }  \sigma, \Pi_{R^{\prime\prime}}(\sigma)}  \nonumber \\+  \sum_{\substack{x \in \{a_{l+1}, \cdots, a_m\} \\ y \in \{a_{l+1},\ldots, a_m\}} } \mathbbm{1}_{(x,y)\text{ discordant for }  A_R(\sigma), \Pi_{R^{\prime\prime}}(A_R(\sigma))}	 \nonumber \\ = \binom{\beta}{2}  \nonumber \, .
	\end{align}
	By analogous reasoning, we have 
	\begin{align}
	\sum_{\substack{x \in \{a_{l+1}, \cdots, a_m\} \\ y \in \{b_1,\ldots, b_q\}} } \mathbbm{1}_{(x,y)\text{ discordant for }  \sigma, \Pi_{R^{\prime\prime}}(\sigma)}  \nonumber \\+ \sum_{\substack{x \in \{a_{l+1}, \cdots, a_m\} \\ y \in \{b_1,\ldots, b_q\}} } \mathbbm{1}_{(x,y)\text{ discordant for }  A_R(\sigma), \Pi_{R^{\prime\prime}}(A_R(\sigma))}  \nonumber \\= (\alpha - \beta)\beta  \nonumber \, .
	\end{align}
	Altogether, these statements yield the result of the lemma.
\end{proof}

\begin{proof}{of Proposition~\ref{prop:negcov}}
\noindent \textbf{Case:} $\sigma_0\in S_n$ be the fixed permutation, then
\begin{align}
\displaystyle\text{Cov}\left(d(\sigma,\sigma_0), d(\anti{\sigma}{R},\sigma_0)\right)&<0\nonumber.
\end{align}
This holds true since

\noindent $d(\anti{\sigma}{R},\sigma_0)={n \choose 2}-d(\sigma,\sigma_0), \forall \sigma\in S_n$, $\forall n \in \mathbb{N}$ by Lemma~\ref{lem:antitheticfull}. 
\noindent \textbf{Case} $\mathbf{\emptyset \subset R}$: Let $\sigma_0\in R$ we have that

\noindent $d(\anti{\sigma}{R},\sigma_0)={n-k\choose 2}$
$-d(\sigma,\sigma_0)$  $\forall \sigma_0\in R$
by Lemma~\ref{lem:antitheticdistancepartial}. 
\end{proof}

 In general, if $\sigma_0\notin R$, by Lemma~\ref{lem:partialrelationanti}, $d(\anti{\sigma}{R},\sigma_0)=$
 
\noindent$d(\sigma, \sigma_0) + \binom{n-k}{2} - 2d(\sigma, \Pi_{R_i}(\sigma_0))$.

After proving all the relevant Lemmas, we now present our main result regarding antithetic samples, namely, that this scheme provides negatively correlated pairs of samples.

\begin{theorem}\label{varantiker}
 	Let the antithetic kernel estimator be evaluated at a pair of partial rankings $R_i, R_j$
where $(\sigma_n)_{n=1}^N \sim \text{Unif}(R_i)$, $(\tau_m)_{m=1}^M \sim \text{Unif}(R_j)$, $N,M$ are the number of \emph{pairs} of samples. If we have $\widetilde{\sigma}_n = A_{R_i}(\sigma_n)$ and $\widetilde{\tau}_m = A_{R_j}(\tau_m)$ for all $m,n$, it corresponds to the antithetic case. If we have $(\widetilde{\sigma}_n)_{n=1}^N \sim \text{Unif}(R_i)$, $(\widetilde{\tau}_m)_{m=1}^M \sim \text{Unif}(R_j)$ independently, it corresponds to the i.i.d. case. Then, the asymptotic variance of the estimator from Equation~\eqref{antikerest} is lower in the antithetic case than in the i.i.d. case.
 \end{theorem}
\begin{proof}
 It has been shown previously that the antithetic kernel estimator is unbiased (in the off-diagonal case), so showing that it has lower MSE in the antithetic case is equivalent to showing that its second moment is smaller in the antithetic case than in the i.i.d. case. The second moment is given by
\begin{gather*}
\mathbb{E}\bigl[\widehat{K}(R_i, R_j)^2\bigr] \\
= \mathbb{E}\Bigl[ \bigl(\frac{1}{4NM}\sum_{n=1}^N \sum_{m=1}^M \bigl( K(\sigma_n, \tau_m) \\+ K(\widetilde{\sigma}_n, \tau_m) +K(\sigma_n, \widetilde{\tau}_m) + K(\widetilde{\sigma}_n, \widetilde{\tau}_m) \bigr)\bigr)^2 \Bigr] \\
= \frac{1}{16M^2N^2} \sum_{n,n^\prime=1}^N \sum_{m,m^\prime=1}^M \mathbb{E}\Bigl[ \bigl( K(\sigma_n, \tau_m) + K(\widetilde{\sigma}_n, \tau_m) \\+K(\sigma_n, \widetilde{\tau}_m) + K(\widetilde{\sigma}_n, \widetilde{\tau}_m)\bigr) \times 
  \bigl( K(\sigma_{n^\prime}, \tau_{m^\prime}) + \\ K(\widetilde{\sigma}_{n^\prime}, \tau_{m^\prime})+K(\sigma_{n^\prime}, \widetilde{\tau}_{m^\prime}) + K(\widetilde{\sigma}_{n^\prime}, \widetilde{\tau}_{m^\prime}) \bigr) \Bigr]\ \, .
\end{gather*}
We identify three types of terms in the above sum: (i) those where $n \not= n^\prime$ and $m \not= m^\prime$; (ii) those where $n=n^\prime$ but $m\not=m^\prime$, or $m=m^\prime$ but $n\not=n^\prime$; (iii) those where $n=n^\prime$ and $m=m^\prime$.

We remark that in case (i), the 16 terms that appear in the summand all have the same distribution in the antithetic and i.i.d. case, so terms of the form (i) contribute no difference between antithetic and i.i.d.. There are $\mathcal{O}(N^2 M + M^2 N)$ terms of the form (ii), and $\mathcal{O}(NM)$ terms of the form (iii). We thus refer to terms of the form (ii) as cubic terms, and terms of the form (iii) as quadratic terms. We observe that due to the proportion of cubic terms to quadratic terms diverging as $N,M \rightarrow \infty$, it is sufficient to prove that each cubic term is less in the antithetic case than the i.i.d. case to establish the claim of lower MSE.

Thus, we focus on cubic terms. Let us consider a term with $n=n^\prime$ and $m\not=m^\prime$. The term has the form
\begin{gather}
	\mathbb{E}\Bigg\lbrack  \bigg( K(\sigma_n, \tau_m) + K(\widetilde{\sigma}_n, \tau_m) + K(\sigma_n, \widetilde{\tau}_m) + K(\widetilde{\sigma}_n, \widetilde{\tau}_m) \bigg) \times \nonumber \\
	 \bigg( K(\sigma_{n}, \tau_{m^\prime}) + K(\widetilde{\sigma}_{n}, \tau_{m^\prime}) + K(\sigma_{n}, \widetilde{\tau}_{m^\prime}) + K(\widetilde{\sigma}_{n}, \widetilde{\tau}_{m^\prime}) \bigg) \Bigg\rbrack \nonumber \, .
\end{gather}
Of the sixteen terms appearing in the expectation above, there are only two distinct distributions they may have. The two types of terms are given below:
\begin{align}\label{eq:term1}
\mathbb{E}\left\lbrack K(\sigma_n, \tau_m) K(\sigma_n, \tau_{m^\prime}) \right\rbrack \, ,
\end{align}
and
\begin{align}\label{eq:term2}
\mathbb{E}\left\lbrack K(\sigma_n, \tau_m) K(\widetilde{\sigma}_n, \tau_{m^\prime}) \right\rbrack \, .
\end{align}
Terms of the form in Equation \eqref{eq:term1} have the same distribution in the antithetic and i.i.d. cases, so we can ignore these. However, terms of the form in Equation \eqref{eq:term2} have differing distributions in these two cases, so we focus in on these. We deal specifically with the case where $K(\sigma, \tau) = \exp(-\lambda d(\sigma, \tau))$, so we may rewrite the expression in Equation \eqref{eq:term2} as
\begin{align}\label{eq:term2dist}
\mathbb{E}\left\lbrack \exp(-\lambda(d(\sigma_n, \tau_m) +d(\widetilde{\sigma}_n, \tau_{m^\prime}))) \right\rbrack \, .
\end{align}
We now decompose the distances $d(\sigma_n, \tau_m)$, $d(\widetilde{\sigma}_n, \tau_{m^\prime})$ using the series of lemmas introduced before. First, we use Lemma \ref{lem:decomp} to write 
\begin{align}
	&d(\sigma_n, \tau_m) = d(\sigma_n, \Pi_{R_1}(\tau_m)) + d(\Pi_{R_1}(\tau_m), \tau_m) \, , \nonumber \\
	&d(\widetilde{\sigma}_n, \tau_{m^\prime}) = d(\widetilde{\sigma}_n, \Pi_{R_1}(\tau_{m^\prime})) + d(\Pi_{R_1}(\tau_{m^\prime}), \tau_{m^\prime})\label{eq:firstdecomposition} \, .
\end{align}
We give a small example illustrating some of the variables at play in this decomposition in Figure \ref{fig:proofillustration}.

\begin{figure}
	\centering
	\includegraphics[keepaspectratio, width=.4\textwidth]{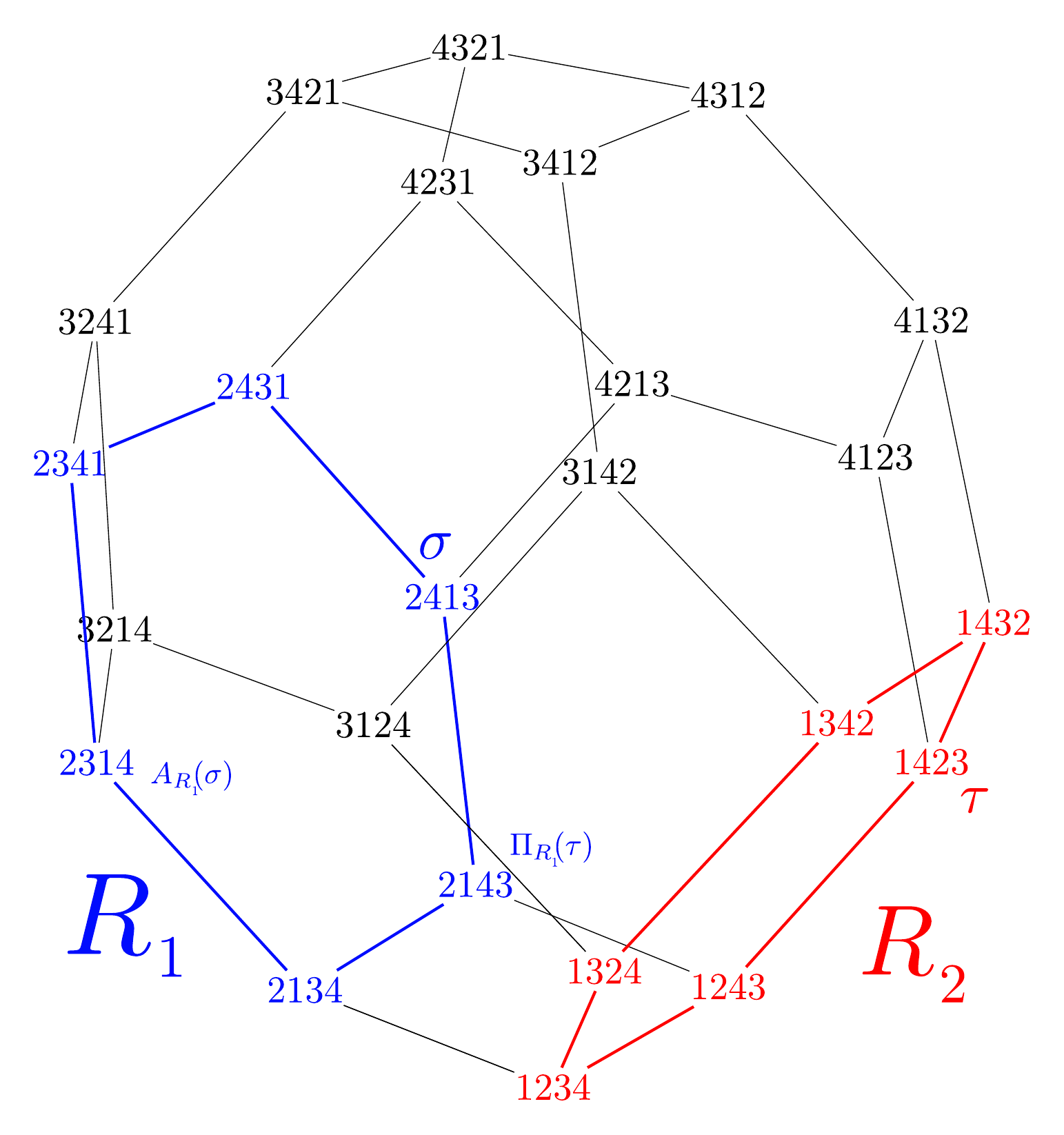}
	\caption{An example of the variables appearing in the decomposition in Equation \eqref{eq:firstdecomposition}.}
	\label{fig:proofillustration}
\end{figure}

Now, writing $R_3 \subseteq R_1$ for the partial ranking described by Lemma \ref{lem:unifOnSmallerRanking}, we have that $\Pi_{R_1}(\tau_m), \Pi_{R_1}(\tau_{m^\prime}) \overset{\mathrm{i.i.d.}}{\sim} \mathrm{Unif}(R_3)$. Therefore, the distances in Equation \eqref{eq:firstdecomposition} may be decomposed further:
\begin{align}
	d(\sigma_n, \tau_m) = d(\sigma_n, \Pi_{R_3}(\sigma_n))& \nonumber \\+ d(\Pi_{R_3}(\sigma_n), \Pi_{R_1}(\tau_m)) + d(\Pi_{R_1}(\tau_m), \tau_m) \, , \nonumber \\
	d(\widetilde{\sigma_n}, \tau_{m^\prime}) = d(\widetilde{\sigma}_n, \Pi_{R_3}(\widetilde{\sigma}_n))& \nonumber\\ + d(\Pi_{R_3}(\widetilde{\sigma}_n), \Pi_{R_1}(\tau_{m^\prime})) \nonumber \\+ d(\Pi_{R_1}(\tau_{m^\prime}), \tau_{m^\prime})\label{eq:seconddecomposition} \, .
\end{align}
We now consider each term, and argue as to whether the distribution is different in the antithetic and i.i.d. cases, recalling that in the i.i.d. case, $\widetilde{\sigma}_n$ is drawn from $R_1$ independently from $\sigma_n$, whilst in the antithetic case, $\widetilde{\sigma}_n = A_{R_1}(\sigma_n)$.

\begin{itemize}
	\item Each of the terms $d(\Pi_{R_1}(\tau_{m}), \tau_{m})$ and 
	
	\noindent $d(\Pi_{R_1}(\tau_{m^\prime}), \tau_{m^\prime})$ have the same distribution under the i.i.d. case and antithetic case. Further, in both cases, $d(\Pi_{R_1}(\tau_{m}), \tau_{m})$ is independent of $\Pi_{R_1}(\tau_m)$, and $d(\Pi_{R_1}(\tau_{m^\prime}), \tau_{m^\prime})$ is independent of $\Pi_{R_1}(\tau_{m^\prime})$, so these two terms are independent of all others appearing in the sum in both cases.
	\item Each of the terms $d(\Pi_{R_3}(\sigma_n), \Pi_{R_1}(\tau_m))$ and
	
	\noindent $d(\Pi_{R_3}(\widetilde{\sigma}_n), \Pi_{R_1}(\tau_{m^\prime}))$ have the same distribution under the i.i.d. case and the antithetic case, and are independent of all other terms in both cases.
	\item We deal with the terms $d(\sigma_n, \Pi_{R_3}(\sigma_n))$ and
	
	\noindent $d(\widetilde{\sigma}_n, \Pi_{R_3}(\widetilde{\sigma}_n))$ using Lemma \ref{lem:projAndAnti}. More specifically, under the i.i.d. case, these two distances are clearly i.i.d.. However, under the antithetic case, the lemma tells us that the sum of these two distances is equal to the mean under the distribution of the i.i.d. case almost surely. Thus, in the antithetic case, this random variable has the same mean as in the i.i.d. case, but is more concentrated (strictly so iff $d(\sigma_n, \Pi_{R_3}(\sigma_n))$ is not a constant almost surely, which is the case iff $R_1 \not= R_3$).
\end{itemize}

Thus, $d(\sigma_n, \tau_m) + d(\widetilde{\sigma_n}, \tau_{m^\prime})$ has the same mean under the i.i.d. and antithetic cases, but is strictly more concentrated when $R_1 \not= R_3$ This holds true iff the partial rankings $R_1$ and $R_2$ do not concern exactly the same set of objects. Thus, by a conditional version of Jensen's inequality, since $\exp(-\lambda x )$ is strictly convex as a a function of $x$, we obtain the variance result.

\end{proof}
\subsection{Antithetic kernel estimator and kernel herding}

In this section, having established the variance-reduction properties of antithetic samples in the context of Monte Carlo kernel estimation, we now explore connections to kernel herding \citep{CheWelSmo10}.

\begin{theorem}\label{thm:herding} The antithetic variate construction of Theorem~\ref{anti} is equivalent to the optimal solution for the first two steps of a  kernel herding procedure in the space of permutations.
\end{theorem}
\begin{proof}
	Let $R$ be a partial ranking of $n$ elements. We calculate the sequence of herding samples from the uniform distribution $p(\cdot | R)$ over full rankings consistent with $R$ associated with the exponential semimetric kernel $K(\sigma, \sigma^\prime) = \exp(-\lambda d(\sigma, \sigma^\prime))$, for a metric $d$ of negative definite type.
	Following \cite{CheWelSmo10}, we note that the herding samples from $p(\cdot | R)$ associated with the kernel $K$, with RKHS embedding $\phi : S_n \rightarrow \mathcal{H}$, are defined iteratively by
	\[
	\sigma_T = \arg \min_{\sigma_T} \left\|\mu_p - \frac{1}{T} \sum_{t=1}^T \phi(\sigma_t) \rbrack \right\|_{\mathcal{H}}^2 \text{\ \ for\ } T=1,\ldots \, ,
	\]
	where $\mu_p$ is the RKHS mean embedding of the distribution $p$. Since $p$ is uniform over its support, any ranking $\sigma$ in the support of $p(\cdot | R)$ is a valid choice as the first sample in a herding sequence. Given such an initial sample, we then calculate the second herding sample, by considering the herding objective as follows:
	\begin{align}
	\left\|\mu_p - \frac{1}{2} \sum_{t=1}^2 \phi(\sigma_t) \rbrack \right\|_{\mathcal{H}}^2 
	& = \| \mu_p \|_\mathcal{H}^2 - \sum_{t=1}^2 \frac{1}{|R|} \sum_{\sigma \in R} K(\sigma_t, \sigma) \nonumber \\
	 &+ \frac{1}{4}\bigl( K(\sigma_1, \sigma_1) + 2K(\sigma_1, \sigma_2) \nonumber \\
	 &+ K(\sigma_2, \sigma_2) \bigr)  
	\end{align}
	which as a function of $\sigma_2$, is equal to $2K(\sigma_1, \sigma_2) = 2\exp(-\lambda d(\sigma_1, \sigma_2))$, up to an additive constant. Thus, selecting $\sigma_2$ to minimize the herding objective is equivalent to maximizing $d(\sigma_1, \sigma_2)$, which is exactly the definition of the antithetic sample to $\sigma_1$.
\end{proof}

After this result, one would like to do a herding procedure for more than two steps. However, the solution is not the same as picking $k$ herding samples simultaneously. Specifically, the following counterexample, illustrated in Figure \ref{fig:counterexample}, clearly shows why. The left plot shows the result of solving the herding objective for 2 samples -- the result is an antithetic pair of samples for the region $R$. If a third sample is selected greedily, with these first two samples fixed, it will yield a different result than if the herding objective is solved for $3$ samples simultaneously, as illustrated on the right of the figure.

\begin{figure}[h]
\centering
\includegraphics[keepaspectratio, width=0.4\textwidth]{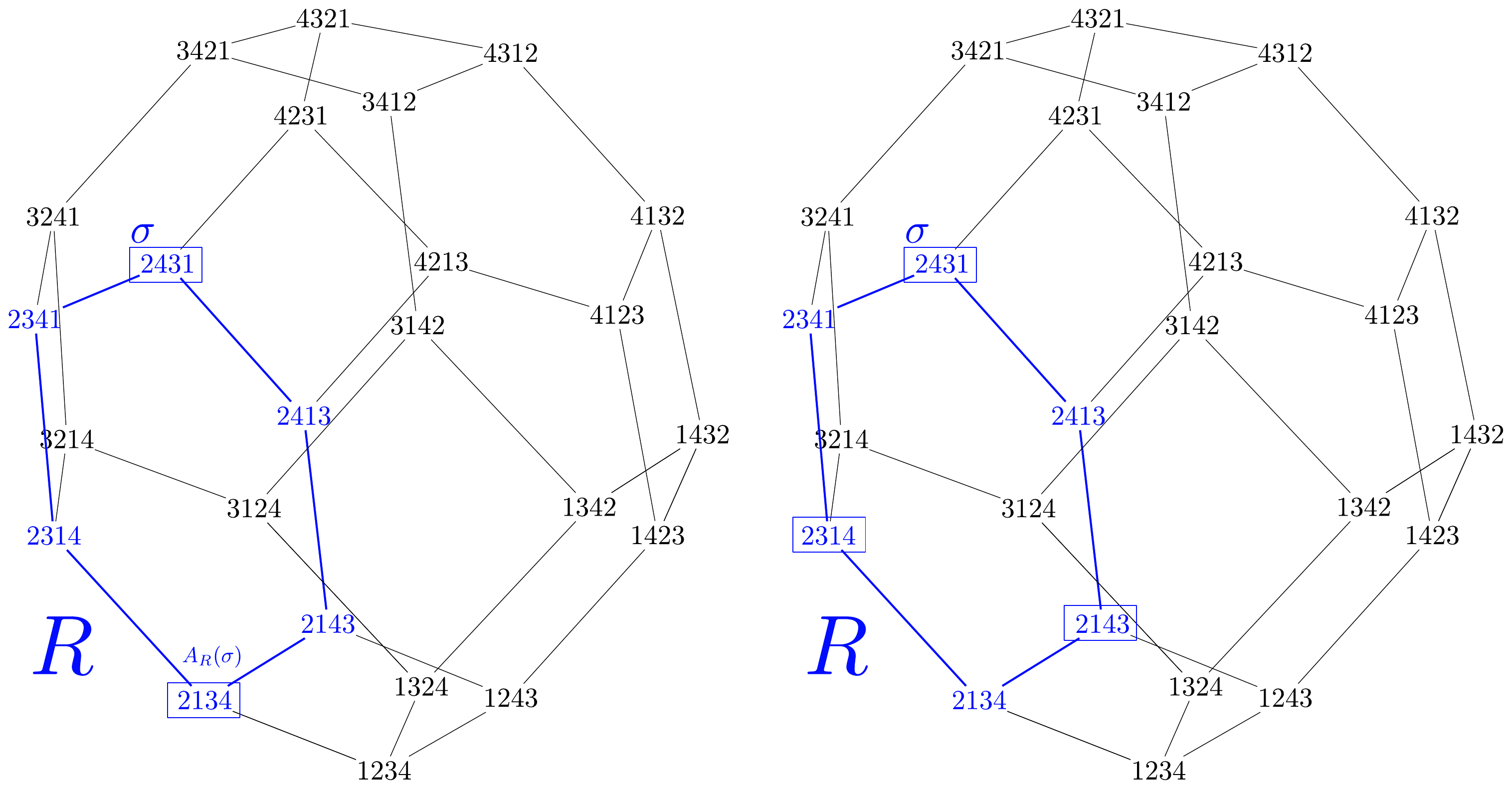}\caption{Samples from the region $R$, illustrating the difference between solving the herding objective greedily, and solving for all samples simultaneously.}
\label{fig:counterexample}
\end{figure}

\begin{rmk}

Theorem~\ref{thm:herding} says that if we first pick a point uniformly at random from $R$, then put it into the herding objective and then select the second deterministically to minimise the herding objective this is equivalent to the antithetic variate construction of Definition~\ref{def:antithetic}. Alternatively, we could pick the second point uniformly at random from $R$, independently from the first point. This second scheme will produce a higher value of the herding objective on average.

\end{rmk}

Once we have constructed two estimators for Kernel matrices we present some experiments to asses their performance in the next section.

\section{Experiments}

In this section, we use the Monte Carlo and antithetic kernel estimators for a variety of machine learning unsupervised and supervised learning tasks: a nonparametric hypothesis test, an agglomerative clustering algorithm and a Gaussian process classifier. 

Definition~\ref{def:antithetic} states the antithetic permutation construction with respect to a given permutation for Kendall's distance. In order to consider partial rankings data, we should respect the observed preferences when obtaining the antithetic variate. The pseudocode from Algorithm~\ref{alg:CompleteRankingUniformCoupled} corresponds to the algorithmic description for sampling an antithetic permutation and simultaneously respecting the constraints imposed by the observed partial ranking. Namely, the antithetic permutation has the observed preferences fixed in the same locations as the original permutation and only reverses the unobserved locations. This corresponds to maximising the Kendall distance between the permutation pair while respecting the constraints and ensures that both permutations have the right marginals as stated in Remark~\ref{rem:topkantithetic} and Lemma~\ref{lem:antithetic-marginal}.
\begin{algorithm}[H]
	\caption{SampleAntitheticConsistentFullRankings}
	\label{alg:CompleteRankingUniformCoupled}
	\begin{algorithmic}
		\STATE {\bfseries Input:} top$-k$ partial ranking $i_1 \succ i_2 \succ \cdots \succ i_k$, degree $n$ 
		\STATE {\bfseries Returns:} two full rankings $\sigma_1, \sigma_2$ consistent with the given partial ranking
		\STATE Set $\sigma_1(l) = \sigma_2(l) = i_l$ for $l=1,\ldots,k$
		\STATE Obtain a random ordering $j_1,\ldots,j_{n-k}$ of the remaining items $\{1,\ldots,n\} \setminus \{ i_1,\ldots,i_k\}$
		\STATE Let $b_1 < \cdots <b_{n-k}$ be the ordering of $\{1,\ldots,n\} \setminus \{i_1,\ldots,i_k\}$
		\STATE Set $\sigma_1(b_l) = j_l$ for $l=1, \ldots, n-k$
		\STATE Set $\sigma_2(b_l) = j_{n-k-l+1}$ for $l=1,\ldots,n-k$
		\STATE Return $\sigma_1, \sigma_2$
	\end{algorithmic}
\end{algorithm}
\subsection{Datasets\label{datasetsdesc}}

\textbf{Synthetic data set.} The synthetic dataset for the in the nonparametric hypothesis test experiment, where the null hypothesis is $H_0:P=Q$ and the alternative is $H_1:P\neq Q$, is the following: the dataset from the $P$ distribution is a mixture of Mallows distributions \citep{Dia88} with the Kendall and Hamming distances.  The central permutations are given by the identity permutation and the reverse of the identity respectively, with lengthscale equal to one. The dataset from the $Q$ distribution is a sample from the uniform distribution over $S_n$, where $n = 6$.

\textbf{Sushi dataset.} This dataset contains rankings about sushi preferences given by 5000 users~\citep{KamAka09}. The users ranked 10 types of sushi and the labels correspond to the user's region (East Japan or West Japan for the Gaussian process classifier and ten different regions for the agglomerative clustering task).

\subsection{Agglomerative clustering}

In this experiment, we used both the full and a censored version of the sushi dataset from Section~\ref{datasetsdesc}. We used various distances for permutations to compute the estimators for the semimetric matrix between pairs of partial rankings subsets.  In order to compute our estimators, we censored the dataset by storing the $topk=4$ partial rankings per user. The Monte Carlo and antithetic kernel estimators were used to obtain negative type semimetric matrices using the relationship from Equation~\eqref{def:semimetneg} in the following way: 
\begin{align}
\widehat{D(R,R')^2}&=\widehat{K}(R,R)+\widehat{K}(R',R')-2\widehat{K}(R,R')\nonumber.
\end{align}
This matrices were then used as an input to the average linkage agglomerative clustering algorithm \citep{DudHar73}. The tree purity measure is reported, it provides way to asses the tree produced by the agglomerative clustering algorithm. It can be computed in the following way: when a dendrogram and all correct labels are given, pick uniformly at random two leaves which have the same label c and find the smallest subtree containing the two leaves. The dendrogram purity is the expected value of $\frac {\#\text{leaves with label c in subtree}}{\#\text{leaves in the subtree}}$ per class. If all leaves in the class are contained in a pure subtree, the dendrogram purity is one. Hence, values close to one correspond to high quality trees. 
\vspace{-6mm}
\begin{table}[H]
       \scalebox{0.75}{
        \setlength\tabcolsep{2pt}
\begin{tabular}{| l | c | c | c | c | c|}
\hline
&&&Semiexp&Semiexp&Semiexp\\
&Kendall&Mallows&Hamming&Cayley&Spearman\\
\hline
$K$ Average&0.83&0.75&0.81&0.72&0.81\\
\hline 
$\widehat{K}$ Average&0.78(0.052)&0.79(0.058)&0.79(0.063)&0.82(0.040)&0.78(0.062)\\
$\widehat{K}^{a}$ Average&NA&0.77(0.050)&NA&NA&NA\\
\hline
\end{tabular}}\caption{Tree purities for the sushi dataset using a subsample of 100 users with the full Gram matrix $K$, a censored dataset of $topk=4$ partial rankings for the vanilla Monte Carlo estimator $\widehat{K}$ and the antithetic Monte Carlo estimator $\widehat{K}^{a}$, with $n_{mc}=20$ Monte Carlo samples. Tree cut at $k=10$ clusters. The median distance criterion was used to select the inverse of the lengthscale for the semimetric exponential kernels.
\label{tablenum}}
\end{table}

In Table~\ref{tablenum}, the true and estimated purities using the full rankings and the partial rankings datasets are reported. We assumed that the true labels are given by the user's region, there are ten different possible regions.  The true purity corresponds to an agglomerative clustering algorithm using the Gram matrix obtained from the full rankings. We can compute the Gram matrix for the full rankings because we have access to all of the users' rankings over the ten different types of sushi. The antithetic Monte Carlo estimator outperforms the vanilla Monte Carlo estimator in terms of average purity since it is closer to the true purity. It also has a lower standard deviation when estimating the marginalised Mallows kernel.

\subsection{Nonparametric hypothesis test with MMD}
Let $P$ and $Q$ be probability distributions over $S_n$, the null hypothesis is $H_0: P = Q$ versus $H_1:P \neq Q$ using samples $\sigma_1,\hdots,\sigma_{n}\stackrel{\text{i.i.d.}}{\sim} P$ and $\sigma'_1,\hdots,\sigma'_{m}\stackrel{\text{i.i.d.}}{\sim} Q$. We can estimate a pseudometric between $P$ and $Q$ and reject $H_0$ if the observed value of the statistic is large. The following is an unbiased estimator of the $MMD^2$ \citep{GreBorMalSchoSmo12}
\begin{align}\label{mmd2}
\widehat{MMD^2}(P,Q)=\frac{1}{m(m-1)}\sum_{i=1}^m\sum_{j\neq i}^mK(\sigma_i,\sigma_j)\nonumber\\+\frac{1}{n(n-1)}\sum_{i=1}^n\sum_{j\neq i}^n K(\sigma'_i,\sigma'_j)\nonumber\\-\frac{2}{nm}\sum_{i=1}^m\sum_{j\neq i}^nK(\sigma_i,\sigma'_j).
\end{align}
This statistic depends on the chosen kernel as can be seen in Equation~\eqref{mmd2}. If the kernel is characteristic~\citep{SriFukLan11}, then the $MMD^2$ is a proper metric over probability distributions.
Analogously, we can compute an MMD squared estimator for partial rankings sets, such that $R_1,\hdots,R_{n}\stackrel{\text{i.i.d.}}{\sim} P$ and $R'_1,\hdots,R'_{m}\stackrel{\text{i.i.d.}}{\sim} Q$, in the following way
\begin{align}\label{mmdsquared}
\widehat{MMD^2}(P,Q)=\frac{1}{m(m-1)}\sum_{i=1}^m\sum_{j\neq i}^m\hat{K}(R_i,R_j)\nonumber\\+\frac{1}{n(n-1)}\sum_{i=1}^n\sum_{j\neq i}^n \hat{K}(R'_i,R'_j)\nonumber\\-\frac{2}{nm}\sum_{i=1}^m\sum_{j\neq i}^n\hat{K}(R_i,R'_j).
\end{align}

We used the synthetic datasets for $P$ and $Q$ described in Section~\ref{datasetsdesc} to asses the performance of the Monte Carlo and antithetic kernel estimators in a nonparametric hypothesis test. The datasets consist of rankings over $n=10$ objects and we censored them to obtain top$-k$ partial rankings with $k=3$. We then computed the MMD squared statistic for the samples using the samples from the two populations. Since the non-asymptotic distribution of the statistic from Equation~\eqref{mmdsquared}  is not known, we performed a permutation test~\citep{AlbJimMu07}  in order to estimate consistently the null distribution and compute the $p-$value. We did this repeatedly as we varied the number of observations for a fixed number of Monte Carlo samples to see the effect of the sample size in the p-value computations. Specifically, Figure~\ref{mmdfig} and Table~\ref{mmdtable} show how the p-value computed with the antithetic kernel estimator has lower variance as we vary the number of observations in our dataset. Both p-values converge to zero since the samples from both populations come from different distributions. In Table~\ref{mmdtable} we report the standard deviations of the estimated p-values. The p-value obtained with the antithetic kernel estimator has lower variance accross all sample sizes.

\begin{figure}
\includegraphics[scale=0.5]{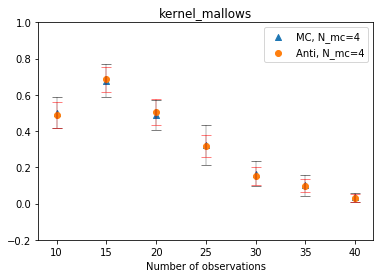}\caption{Mean p-values ($y-$axis) v.s. number of datapoints in synthetic dataset ($x-$axis)\label{mmdfig}}
\end{figure}

\begin{table}[H]
       \scalebox{0.75}{
        \setlength\tabcolsep{2pt}
\begin{tabular}{| c | c | c | c | c| c | c|c|}
\hline
\# obs &10&15&20&25&30&35&40\\
\hline
Monte Carlo &0.0853&  0.0910&0.0830&  0.1109&  0.0677&  0.0596&  0.0236\\
\hline
Antithetic & \textbf{0.0706}&\textbf{0.0663}&\textbf{0.0712}&\textbf{0.0594}&\textbf{0.0502}&\textbf{0.0363}&\textbf{0.0222}\\
\hline
\end{tabular}}\caption{
Standard deviations for $p$-values computed with the Monte Carlo and antithetic estimators \label{mmdtable}}
\end{table}

\subsection{Gaussian process classifier}

In this experiment, two different kernels were used to compute the estimators for the Gram matrix between different pairs of partial rankings subsets. The matrix was then provided as the input to a Gaussian process classifier~\citep{Nea98}. The Python library~\citet{gpy2014} was extended with custom kernel classes for partial rankings which compute both the Monte Carlo and antithetic kernel estimators for partial rankings subsets. Previously, it was only possible to do pointwise evaluations of kernels but our implementation allows to compute the kernels over pairs of partial ranking subsets by storing the sets in a tensor first.

For the Mallows kernel, we used the median distance heuristic~\citep{TakLeSeSmo06, SchSmo02} with the Kendall distance to compute the bandwidth parameter and a scale parameter of 9.5. We performed a grid search over different values of the scale parameter and picked the one that had the largest classification accuracy for the test set. 

\begin{table}[h]\scalebox{0.8}{

\begin{tabular}{| c | c | c | c | }
\hline
&Test accuracy&Train ave-loglik &Test ave-loglik\\
\hline
Mallows&$n_{obs}=100$&&\\
\hline
Full model&0.9&-0.2070&-0.5457\\
MC&0.74(0.016)&-0.2486(0.005)&-0.563(0.020)\\
Antithetic&0.75 (0)&-0.262(0.001)&-0.573(0.002)\\
\hline
Gaussian &$n_{obs}=50$&&\\
\hline
Full model&0.75&-0.2215&-0.7014\\
MC&0.72(0.048)&-0.2890(0.0245)&-0.5737(0.043)\\
Antithetic&NA&NA&NA\\
\hline
Kendall&$n_{obs}=100$&&\\
\hline
Full model&0.7&-0.311(3.01$\times10^{-6}$)&-0.597(3.5$\times10^{-6}$)\\
MC&0.66(0.037)&-0.3575(0.008)&-0.7063(0.052)\\
Antithetic&NA&NA&NA\\
\hline
\end{tabular}}\caption{Averaged over 10 runs with 4 Monte Carlo samples per run, using the inverse of the median distance as the lengthscale, $n=10,topk=6$.\label{tablegp} }
\end{table}
 In Table \ref{tablegp}, the results of running the Gaussian process classifier are reported using the marginalised Mallows kernel, the marginalised Gaussian kernel and the marginalised Kendall kernel as well as the corresponding estimators. Since the Mallows kernel is based on the Kendall distance, it is a kernel specifically tailored for permutations and it is the best in terms of predictive performance. The Gaussian kernel is a kernel that is suitable for Euclidean spaces and it does not take into account the data type but it still does well. The Kendall kernel does take into account the data type but it performs the worst. The full model corresponds to using the Gram matrix using the full rankings and MC and Antithetic refer to using our proposed estimators. We see how empirically the predictive accuracy obtained with the antithetic kernel estimator has lower variance as expected.

\section{Conclusion}
We addressed the problem of extending kernels to partial rankings by introducing a novel Monte Carlo kernel estimator and explored variance reduction strategies via an antithetic variates construction. Our schemes lead to a computationally tractable alternative to previous approaches for partial rankings data. The Monte Carlo scheme can be used to obtain an estimator of the marginalised kernel with any of the kernels reviewed herein. The antithetic construction provides an improved version of the kernel estimator for the marginalised Mallows kernel. Our contribution is noteworthy because the computation of most of the marginalised kernels grows super-exponentially with respect to the number of elements in the collection, hence, it quickly becomes intractable for relatively small values of the number of ranked items $n$. An exception is the fast approach for computing the convolution kernel proposed by~\citet{JiaoVert15}, which is only valid for Kendall kernel. \citet{Mania16} have shown that the Kendall kernel is not characteristic using non-commutative Fourier analysis to show that it has a degenerate spectrum. For this reason, using other kernels for permutations might be desirable depending on the task at hand.

 One possible direction for future work includes the use of explicit feature representations for traditional random features schemes to further reduce the computational cost of the Gram matrix. Another possible application is to use our method with pairwise preference data where users are not necessarily consistent about their preferences. In this type of data, we could still extract a partial ranking from a given user, then, sample from the space of the corresponding full rankings consistent with this observed partial ranking and obtain our Monte Carlo kernel estimator. This would benefit from our framework because having a partial ranking is in general more informative that having pairwise comparisons or star ratings.

Another natural direction for future work is to develop variance-reduction sampling techniques for a wider variety of kernels over permutations, and to the extend the theoretical analysis of these constructions to discrete graphs more generally.


\subsubsection*{Acknowledgements}
Many thanks to Ryan Adams for insightful discussions. Maria Lomeli and Zoubin Ghahramani acknowledge support from the Alan Turing Institute (EPSRC Grant EP/N510129/1), EPSRC Grant EP/N014162/1, and donations from Google and Microsoft Research.  Arthur Gretton thanks the Gatsby Charitable Foundation for financial support. Mark Rowland
acknowledges support by EPSRC grant EP/L016516/1 for the Cambridge Centre for Analysis.


\bibliographystyle{authordate1}

\bibliography{roughbib}
\newpage
\onecolumn

\appendix

\section{Reproducing kernel Hilbert spaces \label{app:rkhs}}
A reproducing kernel Hilbert space (RKHS) \citep{BerAgn04} over a set $\mathcal{X}$ is a Hilbert space $\mathcal{H}$ consisting of functions on $\mathcal{X}$ such that for each $x\in\mathcal{X}$ there is a function $k_x\in \mathcal{H}$ with the property
\begin{align}
\langle f,k_x \rangle_{\mathcal{H}}&=f(x), \hspace{4mm}\forall f\in \mathcal{H}.
\end{align}
The function $k_x(\cdot)=k(x,\cdot)$ is called the \emph{reproducing kernel} of $\mathcal{H}$ \citep{Aron50}. The space $\mathcal{H}$ is endowed with an inner product $\langle\cdot{,}\cdot\rangle_{\mathcal{H}}$ and a norm can be defined based on it such that $\|f\|_{\mathcal{H}}:=\sqrt{\langle\cdot{,}\cdot\rangle_{\mathcal{H}}}$. In order to be a Hilbert space, it needs to contain all limits of Cauchy sequences, i.e. it has to be complete. In the case of the symmetric group of degree $n$, $\mathcal{X}=S_n$, the space is finite dimensional which guarantees that it is complete. Finally, any symmetric and positive definite function $k_x:\mathcal{X}\times \mathcal{X}\rightarrow\mathbb{R}$ uniquely determines an RKHS. Alternatively, a function $k:\mathcal{X}\times\mathcal{X}\rightarrow\mathbb{R}$ is called a \emph{kernel} if there exists a Hilbert space $\mathcal{H}$ and a map $\phi:\mathcal{X}\rightarrow \mathcal{H}$ such that $\forall x,y \in \mathcal{X}$, $k(x,y)=\langle \phi(x),\phi(y) \rangle_{\mathcal{H}}$. The function $\phi$ is usually referred to as the \emph{feature representation} of $x$. Even though the RKHS induced by the kernel is unique, there can be more than one feature representations that define the same kernel.

\section{Expectation of the Kernel Monte Carlo estimator \label{app:unbiasednessproof}}
\begin{proof}
 For distinct $i,j =1,\ldots, I$, let $\left\{\sigma_n^{(i)}\right\}_{n=1}^{N_i}$ be an independent and identically distributed (i.i.d.) sample from $p(\sigma\mid R_i)$ and $\left\{\sigma_m^{(j)}\right\}_{m=1}^{N_j}$ be an i.i.d. sample from $p(\sigma\mid R_j)$. 

If the weights are uniform,
\begin{align}
\mathbb{E}\left(\widehat{K}(R_i,R_j)\right)&=\frac{1}{N_iN_j}\sum_{n=1}^{N_i}\sum_{m=1}^{N_j}\ \mathbb{E}\left(K(\sigma_n^{(i)},\sigma_m^{(j)})\right)\nonumber \\
\end{align}
By linearity of expectation, since the samples are identically distributed, the expectation in the summand above reduces to
\begin{align}
=\sum_{\sigma\in R_{i}}\sum_{\sigma'\in R_{j}}K(\sigma,\sigma')p(\sigma\mid R_i)p(\sigma'\mid R_j)\nonumber
\end{align}
as required. The diagonal case,
\begin{align}
\mathbb{E}\left(\widehat{K}(R_i,R_i)\right)&=\frac{1}{N_i^2}\left( \sum_{n=1}^{N_i}\sum_{m=1}^{N_i}\mathbb{E}\left(K(\sigma_n^{(i)},\sigma_m^{(i)})\right)\right)\nonumber\\
&=\frac{1}{N_i^2}\sum_{n=1}^{N_i}\mathbb{E}\left[ \sum_{m\neq n}^{N_i}\mathbb{E}\left(K(\sigma_n^{(i)},\sigma_m^{(i)}\mid \sigma_n^{(i)})\right)+\mathbb{E}(K(\sigma_n^{(i)},\sigma_n^{(i)}\mid \sigma_n^{(i)})\right]\nonumber\\
&=\frac{1}{N_i^2}\sum_{n=1}^{N_i}\mathbb{E}\left[ (N_i-1)\mathbb{E}\left(K(\sigma^{(i)},\sigma^{'(i)}\mid \sigma_n^{(i)})\right)+\mathbb{E}\left(K(\sigma^{'(i)},\sigma^{'(i)})\right)\right]\nonumber\\
&=\frac{(N_i-1)}{N_i}\mathbb{E}_{\sigma,\sigma'}\left(K(\sigma^{(i)},\sigma^{'(i)})\right)+\frac{1}{N_i}\mathbb{E}_{\sigma'}\left(K(\sigma^{'(i)},\sigma^{'(i)})\right)\nonumber.
\end{align}

If the weights are non-uniform and are given by the importance sampling weights, namely $w^{(i)}=\frac{p(\sigma\mid R_i)}{q(\sigma\mid R_i)}$, and the expectation is taken with respect to the proposal $q$, then
\begin{align}
\mathbb{E}_q\left(\widehat{K}(R_i,R_j)\right)&=\frac{1}{N_iN_j}\sum_{n=1}^{N_i}\sum_{m=1}^{N_j}\ \mathbb{E}_q\left(\frac{p(\sigma_n^{(i)}\mid R_i)}{q(\sigma_n^{(i)}\mid R_i)}\frac{p(\sigma_m^{(j)}\mid R_i)}{q(\sigma_m^{(j)}\mid R_i)}K(\sigma_n^{(i)},\sigma_m^{(j)})\right)\nonumber 
\end{align}
By linearity of expectation, since the samples are identically distributed, the expectation in the summand above reduces to
\begin{align}
=\sum_{\sigma\in R_{i}}\sum_{\sigma'\in R_{j}}K(\sigma,\sigma')p(\sigma\mid R_i)p(\sigma'\mid R_j)\nonumber
\end{align}
as required. 
\end{proof}
\section{Variance of Kernel Monte Carlo estimator with i.i.d. samples\label{app:varMCK}}
\begin{proof} {The variance of the Kernel Monte Carlo estimator with uniform weights is the following:}
\begin{align*}
	\displaystyle \mathrm{Var}\left[\widehat{K}(R_i,R_j)\right]&=\frac{1}{N_i^2N_j^2}\mathrm{Var}\left[\sum_{n=1}^{N_i} \sum_{m=1}^{N_j} K(\sigma_n^{(i)},\sigma_m^{(j)})\right]\nonumber\\
	&=\frac{1}{N_i N_j^2}\left[ \mathrm{Var}\left( \sum_{m=1}^{N_j} K(\sigma_1^{(i)},\sigma_m^{(j)})\right)\right]
	\end{align*}

If we use the law of total variance, then
\begin{align*}
	 \mathrm{Var}\left( \sum_{m=1}^{N_j} K(\sigma_1^{(i)},\sigma_m^{(j)})\right)&= \mathrm{Var}\left(\mathbb{E}\biggr[\sum_{m=1}^{N_j} K(\sigma_1^{(i)},\sigma_m^{(j)}) \mid \sigma_1^{(i)}\biggl]\right)
	 + \mathbb{E}\left( \mathrm{Var}\biggl[\sum_{m=1}^{N_j} K(\sigma_1^{(i)},\sigma_m^{(j)}) \mid \sigma_1^{(i)}\biggr]\right)\nonumber\\
	 &= {N_j}^2\mathrm{Var}\left(\mathbb{E}\biggl[ K(\sigma_1^{(i)},\sigma_1^{(j)}) \big| \sigma_1^{(i)}\biggr]\right)+N_j\mathbb{E}\left( \mathrm{Var}\biggr[ K(\sigma_1^{(i)},\sigma_1^{(j)}) \big| \sigma_1^{(i)}\biggl]\right)\nonumber\\
	 &=N_j^2\mathrm{Var}\left(\sum_{\sigma'\in R_j} K(\sigma_1^{(j)},\sigma')p(\sigma'\mid R_j) \right)\\
	 &+N_j\mathbb{E}\left(\sum_{\sigma'\in R_j} K(\sigma_1^{(i)},\sigma')^2p(\sigma'\mid R_j)\right)
	 	 \nonumber\\
		 &-N_j\mathbb{E}\left(\biggl(\sum_{\sigma'\in R_j} K(\sigma_1^{(i)},\sigma')p(\sigma'\mid R_j)\biggr)^2\right)
	 	 \nonumber\\
		 &=N_j^2\sum_{\sigma\in R_i}p(\sigma'\mid R_i)\biggl(\sum_{\sigma'\in R_j}p(\sigma'\mid R_j)K(\sigma,\sigma')\biggr)^2\nonumber\\
	&-N_j^2\biggl(\sum_{\sigma\in R_i}\sum_{\sigma'\in R_j}K(\sigma,\sigma')p(\sigma\mid R_i)p(\sigma'\mid R_j)\biggr) ^2\nonumber\\ 
	&+N_j\sum_{\sigma\in R_i}p(\sigma\mid R_i)\left(\sum_{\sigma'\in R_j}p(\sigma'\mid R_j)K(\sigma,\sigma')\right)^2\nonumber\\
	&-N_j\sum_{\sigma'\in R_j}\sum_{\sigma\in R_i} K(\sigma,\sigma')^2p(\sigma'\mid R_j)p(\sigma\mid R_i) \nonumber
\end{align*}

So the variance for the Monte Carlo kernel estimator is given by
\begin{align*}
	\displaystyle \mathrm{Var}\left[\widehat{K}(R_i,R_j)\right]&=\frac{1}{N_i}\biggl[\sum_{\sigma\in R_i}p(\sigma\mid R_i)\biggl(\sum_{\sigma'\in R_j}p(\sigma'\mid R_j)K(\sigma,\sigma')\biggr)^2\nonumber\\
	&-\biggl(\sum_{\sigma\in R_i}\sum_{\sigma'\in R_j}K(\sigma,\sigma')p(\sigma\mid R_i)p(\sigma'\mid R_j)\biggr) ^2\biggr]\nonumber\\
	+&\frac{1}{N_iN_j}\biggl[\sum_{\sigma\in R_i}p(\sigma\mid R_i)\left(\sum_{\sigma'\in R_j}p(\sigma'\mid R_j)K(\sigma,\sigma')\right)^2\nonumber\\
	&-\sum_{\sigma\in R_i}\sum_{\sigma'\in R_j}K(\sigma,\sigma')^2p(\sigma\mid R_i)p(\sigma'\mid R_j)\biggr].
\end{align*}

\end{proof}

\section{Factorial growth of the space of consistent full rankings for a given partial ranking\label{bignumbers}}

\medskip

\begin{table}[H]
						\vskip 0.15in
		\begin{center}
			\begin{small}
				\begin{sc}
					\begin{tabular}{lcccr}
						\hline
						$n/\#\text{pairs}$& 1& 2 &3&4  \\
						\hline
						3    &3 & 1&-& -  \\
						4    &12 &4 &1&-   \\
						5    & 60&20 &5& 1 \\
						6    &360 &120 &30& 6 \\
						7    &2520 & 840&210& 42 \\
						8    &20160 & 6720&1680& 336 \\
						9    & 181440& 60480&15120& 3024  \\
						10    & 1814400& 604800&151200& 30240  \\
						\hline
					\end{tabular}
				\end{sc}
			\end{small}
		\end{center}
		\caption{Table of partial rankings subset cardinalities for a given number of items and number of observed preferences (by overlapping pairs)\label{bignumbers-table}
}
		\vskip -0.1in
	\end{table}

\medskip

\end{document}